\documentclass{article}
\usepackage[letterpaper,top=1in,bottom=1in,left=1in,right=1in]{geometry}

\newif\ifdraft
\draftfalse

\newif\ifoc % Single column?
\octrue

\ifdraft
\usepackage[letterpaper,paperwidth=9in,paperheight=11in,
			top=0.75in,bottom=1in,left=0.125in,right=1.625in,
			footskip=.25in,heightrounded,marginparwidth=1.6in,
			marginparsep=0.05in]{geometry}
\usepackage{xcolor}
\usepackage{xargs} 
\usepackage[textsize=footnotesize]{todonotes}
\newcommandx{\nt}[2][1=]{\todo[linecolor=blue,
			backgroundcolor=blue!10,bordercolor=blue,#1]{#2}}
\newcommandx{\jj}[2][1=]{\todo[linecolor=red,
			backgroundcolor=red!10,bordercolor=red,#1]{{\bf JJ}: #2}}
\def\ct#1{\textcolor{red}{#1}}
\else
\def\ct#1{}
\def\td#1{}
\def\jj#1{}
\fi

\usepackage{tabularx}
\usepackage{amssymb,amsmath,amsthm}
\usepackage{epsfig}
\usepackage{textcomp}
\usepackage{epstopdf}
\usepackage{url}
\usepackage{cite}
\usepackage{wrapfig}
\usepackage{enumerate}
\usepackage{soul}
\usepackage{tipa}
\usepackage[percent]{overpic}
\usepackage{tikz}
\usepackage{pgfgantt}

\usepackage[linesnumbered,ruled]{algorithm2e}
\SetKwInput{KwIn}{Initial condition}
\SetKwInput{KwResult}{Outcome}
\SetAlgorithmName{Algorithm}{List of algorithms}

\newtheorem{theorem}{Theorem}

\newtheorem{proposition}[theorem]{Proposition}
\newtheorem{corollary}[theorem]{Corollary}

\newtheorem{lemma}[theorem]{Lemma}

\makeatletter
\newcommand{\customlabel}[2]{%
\protected@write \@auxout {}{\string \newlabel {#1}{{#2}{}}}}
\makeatother

\usepackage{xspace}

\def\sagalgo{\textsc{SplitAndGroup}\xspace}
\def\sag{\textsc{SaG}\xspace}
\def\isag{i\textsc{SaG}\xspace}
\def\pafalgo{\textsc{PartitionAndFlow}\xspace}
\def\paf{\textsc{PaF}\xspace}

\def\mpp{MPP\xspace}

\usepackage{mathtools,xparse}

\definecolor{purp}{rgb}{0.4, 0.1, 0.6}

\title{Constant-Factor Time-Optimal Multi-Robot Routing on High-Dimensional Grids in 
Mostly Sub-Quadratic Time}% in Arbitrary Dimensions}

\author{Jingjin Yu %
\thanks{Jingjin Yu is with the Department of Computer 
Science, Rutgers University at New Brunswick. E-mails: 
jingjin.yu@cs.rutgers.edu.
}
}

\begin{document}
\maketitle

\begin{abstract}Let $G = (V, E)$ be an $m_1 \times \ldots \times m_k$ grid. 
Assuming that each $v \in V$ is occupied by a robot and a robot may move 
to a neighboring vertex in a step via synchronized rotations along cycles 
of $G$, we first establish that the arbitrary reconfiguration of labeled 
robots on $G$ can be performed in $O(k\sum_i m_i)$ makespan and requires 
$O(|V|^2)$ running time in the worst case and $o(|V|^2)$ when $G$ is 
non-degenerate (in the current context, a grid is degenerate if it is nearly 
one dimensional). The resulting algorithm, \isag, provides average case 
$O(1)$-approximate (i.e., constant-factor) time optimality guarantee. 
When all dimensions are of similar size $O(|V|^{\frac{1}{k}})$, the running 
time of \isag approaches a linear $O(|V|)$. Define $d_g(p)$ as the largest 
distance between individual initial and goal configurations over all robots 
for a given problem instance $p$, building on \isag, we develop the \pafalgo 
(\paf) algorithm that computes $O(d_g(p))$ makespan solutions for arbitrary 
fixed $k \ge 2$, using mostly $o(|V|^2)$ running time. \paf  provides worst 
case $O(1)$-approximation regarding solution time optimality. We note that 
the worst case running time for the problem is $\Omega(|V|^2)$. 
\end{abstract}

\section{Introduction}\label{section:introduction}
%Mention sorting based approach is a waste
We study the time-optimal multi-robot routing or path planning problem on 
$k$ dimensional grids and grid-like settings, with the assumption that each 
vertex of the grid is occupied by a labeled robot, i.e., the robot density 
is maximal. Our work brings several technical breakthroughs: 
\begin{itemize}
\item On a $k \ge 2$ (assuming $k$ is a constant) dimensional grid $G = (V, E)$, 
our algorithm, \isag, improves the running time of the {\em average case} 
$O(1)$-approximate (makespan) time-optimal \sagalgo(\sag) algorithm from 
\cite{yu2017constant} from $O(|V|^3)$ to a sub-quadratic $o(|V|^2)$ for most 
cases and $O(|V|^2)$ in the worst case (when $G$ is degenerate and nearly one 
dimensional). The problem has a worst case time complexity lower bound of 
$\Omega(|V|^2)$.  
\item Define $d_g(p)$ as the largest distance between individual initial 
and goal configurations over all robots for a given problem instance $p$,
building on \isag, we develop the \pafalgo (\paf) algorithm that computes
$O(d_g(p))$ makespan solutions for arbitrary fixed dimension in mostly 
$o(|V|^2)$ time and $O(|V|^2)$ time in the worst case. \paf provides {\em 
worst case} $O(1)$-approximate guarantee on time optimality. We note that 
\paf is developed independently of a key result from 
\cite{demaine2018coordinated} that achieves the same effect for two 
dimensions only. 
\item Certain techniques in our work, which help enable the near optimal 
running time for \isag and \paf, may be of independent interest, including: 
\begin{itemize}
\item We provide a shuffling procedure based on bipartite matching that 
allows the arbitrary redistribution of a group of unlabeled robots on 
arbitrary-dimensional grids (Theorem~\ref{t:kd-shuffle}). 
\item We provide an efficient procedure, also based on matching, that 
decouples an $f > 0$ circulation into $f$ unit circulations on arbitrary
graphs (Theorem~\ref{t:fd}).
\item We establish the existence of $\Omega(d_g^{k-1})$ vertex disjoint paths 
for {\em reshaping} the same amount of flow through a $k$ dimensional grid 
with a side length of $\Theta(d_g)$ (Lemma~\ref{l:vdp-kd}). 
\end{itemize}
\end{itemize}

From the practical standpoint, our results are of significance in multiple 
application domains including robotics and network routing. Particularly, in 
robotics, our results imply that even in highly dense settings, 
if among a group of labeled robots the maximum distance between a robot and 
its goal is of distance $d_g$, then it is possible to compute a routing plan 
that solves the entire problem that requires $O(d_g)$ makespan in only 
quadratic time, assuming that the robots travel at no faster than unit speed. 
Further exploration of the algorithmic insights from our work may lead to more 
optimal coordination algorithms for applications including warehousing 
\cite{WurDanMou08}, automated container port management 
\cite{stahlbock2008operations}, and coordinated aerial flight \cite{tang2018hold}. 
As noted in \cite{demaine2018coordinated}, algorithms like \paf also help 
resolve open questions regarding routing strategies for inter-connected mesh 
networks. Indeed, solving multi-robot routing on grid and grid-like structures 
is equivalent to finding vertex disjoint paths in the underlying network, 
extended over discrete time steps. 

\textbf{Related work}. 
Multi-robot path planning, from both the algorithmic and the application 
perspectives, has been studied extensively 
\cite{ErdLoz86,LavHut98b,GuoPar02,JanStu08,LunBer11,StaKor11,BerSnoLinMan09,
SolHal12,YuLav13STAR,TurMicKum14,ChoLynHutKanBurKavThr05,blm-rvo,
bekris2007decentralized,alonso2015local,knepper2012pedestrian}, covering many 
application domains \cite{HalLatWil00, Nna92, RodAma10, FoxBurKruThr00,DinChaFai01,
GriAke05,MatNilSim95, RusDonJen95,JenWheEva97,tang2018hold}. Multi-robot path 
and motion planning is known to be computationally hard under continuous settings 
\cite{SpiYak84,HopSchSha84}, even when the robots are unlabeled \cite{HeaDem05,SolHal15}. 
While the general multi-robot motion planning problem seems rather difficult 
to tackle, relaxed unlabeled continuous problems are solvable in polynomial time 
even near optimally \cite{TurMicKum14,SolYu15}. 

Restricting our attention to the discrete and labeled setting, in contrast to 
the continuous setting, feasible solutions are more readily computable. Seminal 
work by Kornhauser et al. \cite{KorMilSpi84}, which builds on the work by Wilson 
\cite{Wil74}, establishes that a discrete instance can be checked and solved in 
$O(|V|^3)$ time on a graph $G=(V,E)$. Feasibility test can in fact be completed in 
linear time \cite{AulMonParPer99,GorHas10,YuRus15STAR}. Optimal solutions 
remain difficult to compute in the discrete settings, however, even on planar 
graphs \cite{Yu2016RAL,demaine2018coordinated}. Whereas many algorithms have 
been proposed toward optimally solving the discrete labeled multi-robot path 
planning problems \cite{StaKor11,ShaSteFelStu12,WagChoC11,ferner2013odrm,
sharon2013increasing,boyarski2015icbs,honig2016multi,cohen2016improved,YuLav16TOR}, 
few provide simultaneous guarantees on solution optimality and (polynomial) 
running time. This leads to the development of polynomial time methods that 
also provide these desirable guarantees \cite{yu2017constant,demaine2018coordinated}.

\textbf{Organization}. 
The rest of the paper is organized as follows. In 
Section~\ref{section:preliminary}, we outline the multi-robot path planning 
problem to be solved. In Section~\ref{section:improved-average}, we provide 
an average case $O(1)$-approximate algorithm, \isag, that significantly 
improves an earlier algorithm for the same purpose \cite{yu2017constant}. 
In Section~\ref{section:sketch}, we provide an descriptive outline of the 
key \pafalgo (\paf) algorithm, restricted to the 2D setting, which frequently 
invokes \isag as a subroutine to realize $O(1)$-approximation in the worst case. 
While only the 2D setting is being discussed in this section, we mention that 
the general underlying strategy applies to higher dimensions as well. 
Sections~\ref{section:two-dimensions} and~\ref{section:high-dimensions} are 
then devoted to the details of \paf in 2D and higher dimensions, respectively. 
We conclude with some discussions in Section~\ref{section:discussion}.

\section{Preliminaries}\label{section:preliminary}
Let $G = (V, E)$ be a simple, undirected, and connected graph. 
A set of $n \le |V|$ robots labeled $1$-$n$ may move synchronously on $G$ in a 
collision-free manner described as follows. At integer (time) steps starting 
from $t = 0$, each robot must reside on a unique vertex $v \in V$, 
inducing a {\em configuration} $X_t$ of the robots as an injective map 
$X_t: \{1, \ldots, n\} \to V$, specifying which robot occupies which vertex 
at step $t$ (see Fig.~\ref{fig:problem}). From step $t$ to step $t + 1$, 
a robot may {\em move} from its current vertex to an adjacent 
one under two collision avoidance constraints: {\em (i)} $X_{t+1}$ is 
injective, i.e., each robot occupies a unique vertex, and {\em (ii)} for 
$1 \le i, j \le n$, $i \ne j$, $X_t(i) = X_{t+1}(j) \to X_t(j) \ne 
X_{t+1}(i)$, i.e., no two robots may {\em swap} locations in a single 
step. If all individual robot moves between some $X_t$ and $X_{t+1}$ are 
valid (i.e., collision-free), then $M_t = (X_t, X_{t+1})$ is a valid 
{\em move} for all robots. Multiple such moves can be chained together
to form a sequence of moves, e.g., taking the form of  
$(X_t, X_{t+1}, \ldots, X_{t+t'})$ for some positive integer $t'$. 

\begin{figure}[h]
\begin{center}
\begin{overpic}[width={\ifoc 2.8in \else 2.66in \fi},tics=5]{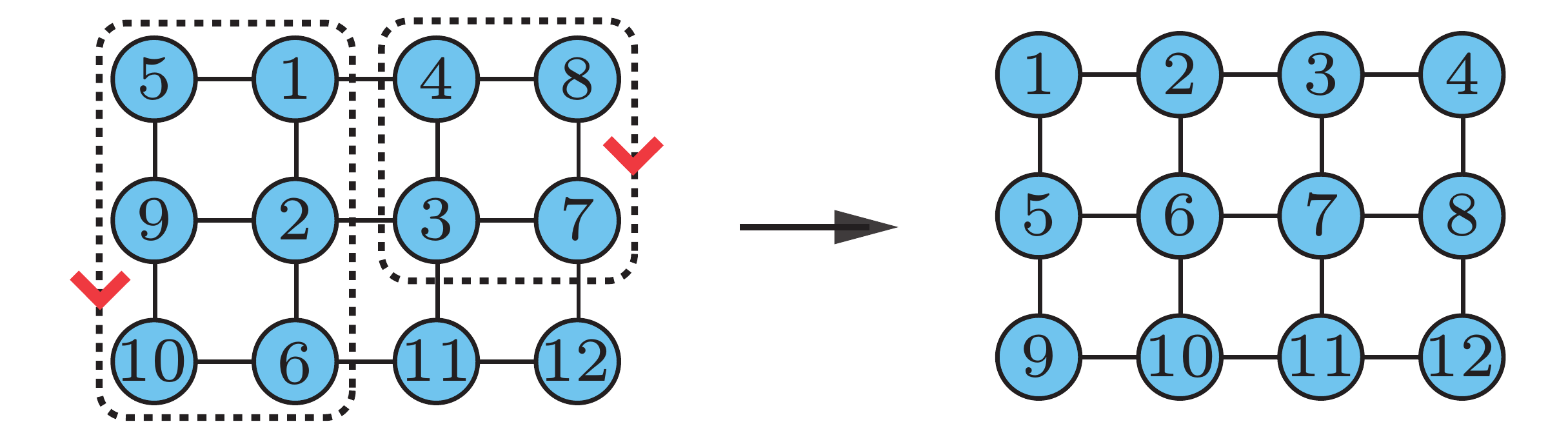}
\ifoc
\put(21.5, -5){{\small (a)}}
\put(77.4, -5){{\small (b)}}
\else
\put(21.5, -3){{\small (a)}}
\put(77.4, -3){{\small (b)}}
\fi
\end{overpic}
\end{center}
\caption{\label{fig:problem} Graph-theoretic formulation of the multi-robot 
path planning problem. (a) A configuration of 12 robots on a $4 \times 3$
grid. (b) A configuration that is reachable from (a) in a single synchronous 
move through simultaneous rotations of robots along two disjoint cycles.}
\end{figure}
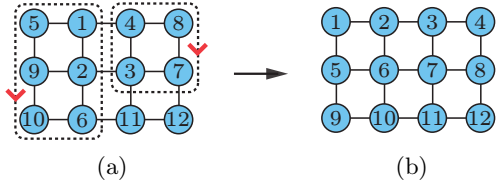

Under this model, a multi-robot path planning problem (\mpp) instance is 
fully specified with a 3-tuple $(G, X_I, X_G)$ in which $X_I = X_0$ and $X_G$ 
are the initial and goal configurations, respectively. 
To handle the most difficult case, we assumed that $n = |V|$, i.e.,
the number of robots is the maximum possible under the model. 
We note that the case of $n' < |V|$ may be reduced to 
the $n = |V|$ case by arbitrarily placing $(|V| - n')$ ``virtual'' robots  
on vertices that are empty as indicated by $X_I$ and $X_G$. An algorithm for 
the $n = |V|$ case is then also an algorithm for the $n' < |V|$ case via the 
reduction. 

For this study, $G$ is assumed to be a $k$-dimensional ($k \ge 2$) 
grid graph, i.e., $G$ is an $m_1 \times \ldots \times m_k$ grid with $|V| = 
\prod_{i=1}^km_i$. For each vertex $v$ of $G$ that is not on the boundary 
of $G$, $v$ is connected to $2k$ other vertices, $2$ in each dimension. 
Without loss of generality, throughout the paper, we always assume that 
$m_1 \ge \ldots \ge m_k \ge 2$ and $|V| \ge 6$ (note that constant sized 
problems can be solved in $O(1)$ makespan through first doing brute force 
search and then direct solution look up, which takes constant time). Such a 
grid graph $G$ is also meant whenever the term {\em grid} is used in the paper 
without further specifications. We say $G$ is {\em degenerate} if $m_1 = 
\Omega(|V|)$, which implies that all other dimensions are of constant sizes, 
i.e., $G$ is mostly one-dimensional. Otherwise, $G$ is {\em non-degenerate}. 
Since the most interesting cases are $k = 2, 3$ due to their relevance in 
applications, these cases are sometimes treated more carefully with additional 
details. 

Given an \mpp instance and a feasible solution, as a sequence of moves $M = 
(X_I=X_0, X_1, \ldots, X_{t_f}=X_G)$ that takes $X_I$ to $X_G$, we define 
the solution's {\em makespan} as the length $t_f$ of the sequence. For an 
\mpp instance $p = (G, X_I, X_G)$, let $d(v_1, v_2)$ denote the distance 
between two vertices $v_1, v_2 \in V$, assuming each edge has unit length. 
We define the {\em distance gap} between $X_I$ and $X_G$ as 
\[
d_g(p) = \max_{1 \le i \le |V|} d(X_I(i), X_G(i)),  
\]
which is an underestimate of the minimum makespan for $p$. The main aim of 
this work is to establish a polynomial time algorithm that computes solutions 
with $O(d_g(p))$ makespan for an arbitrary instance $p$ whose underlying 
grid are of some fixed dimension $k \ge 2$. In other words, the algorithm 
produces, in the worst case, $O(1)$-approximate makespan optimal solutions. 
Note that, on an $m_1 \times \ldots \times m_k$ grid, $d_g(p) \le 
\sum_{1=1}^k (m_i-1)$. 

\section{Improved Average Case $O(1)$-Approximate Makespan Algorithm}\label{section:improved-average}
Our worst case $O(1)$-approximate algorithm makes use of, as a subroutine, 
an average case $O(1)$-approximate algorithm for the same problem that 
improves over the \sagalgo (\sag) algorithm from \cite{yu2017constant}. Main 
properties of \sag are summarized in the following theorem. 

\begin{theorem}[\cite{yu2017constant}]\label{t:ao1}Let $(G, X_I, X_G)$ be an 
\mpp instance with $G = (V, E)$ being an $m_1 \times m_2$ grid. Then, a 
solution with $O(m_1 + m_2)$ makespan can be computed in $|V|^3$ time.
\end{theorem}

To be able to state our improvements over \sag, we briefly describe how \sag 
operates on an $m_1 \times m_2$ grid $G$. \sag recursively splits $G$ into 
halves along a longer dimension. During the first iteration, $G$ is {\em split} 
into two $\frac{m_1}{2} \times m_2$ grids (assuming without loss of 
generality that $m_1$ is even), $G_1$ and $G_2$. Then, all robots 
whose goals belong to $G_2$ will be routed to $G_2$. This will also force 
all robots whose goals belong to $G_1$ to be moved to $G_1$ because $G$ is 
fully occupied. This effectively partitions all robots on $G$ into two 
equivalence classes (those should be in $G_1$ and those should be in $G_2$); 
there is no need to distinguish the robots within each class during the
current iteration. This is the {\em grouping} operation in \sag. 
Fig.~\ref{fig:split} illustrates graphically what is to be achieved in the 
grouping operation in an iteration of \sag. 
\begin{figure}[h]
\begin{center}
\begin{overpic}[width={\ifoc 3.2in \else 3.04in \fi},tics=5]{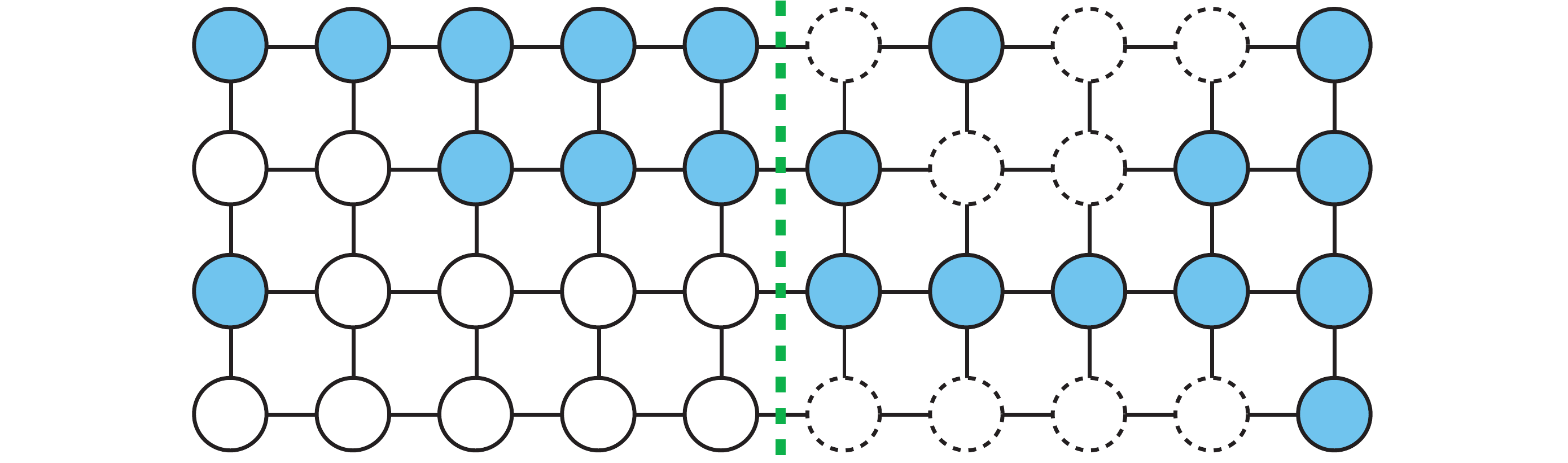}
\end{overpic}
\end{center}
\caption{\label{fig:split} On a $10\times 4$ grid, the shaded robots have goals 
on the right $5\times 4$ grid. The grouping operation of an \sag iteration seeks 
to move the $9$ shaded robots on the left $5\times 4$ grid to exchange with the 
$9$ unshaded robots marked with dashed boundaries on the right $5\times 4$ grid.}
\end{figure}

To be able to move the robots to the desired halves of $G$, it was noted 
\cite{YuLav16TOR} that an exchange of two robots can be realized on a 
$3 \times 2$ grid using a constant number of moves (Fig.~\ref{fig:23}). 
\begin{figure}[h]
\begin{center}
\begin{overpic}[width={\ifoc 3.6in \else 3.42in \fi},tics=5]{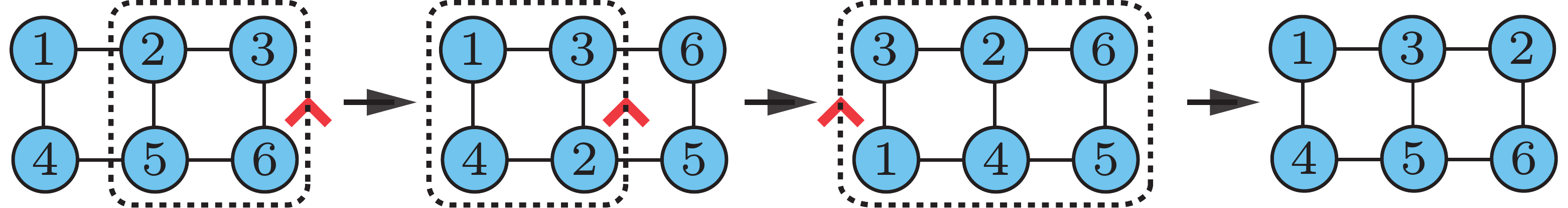}
\end{overpic}
\end{center}
\caption{\label{fig:23} Robots $2$ and $3$ may be ``swapped'' using three 
synchronous moves on a $3 \times 2$ grid. This implies that arbitrary 
configuration on a $3\times 2$ grid can be realized in a constant number of 
moves.}
\end{figure}

The local ``swapping'' primitives can be executed in parallel on $G$, which 
implies Lemma~\ref{l:distribute} as follows. An illustration of the operation
is provided in Fig.~\ref{fig:distribute}.
\begin{lemma}[Lemma 6 in \cite{yu2017constant}]\label{l:distribute} On a 
length $\ell$ path embedded in a grid, a group of indistinguishable robots 
may be arbitrarily rearranged using $O(\ell)$ makespan. Multiple such 
rearrangements on vertex disjoint paths can be carried out in parallel. 
\end{lemma}
\begin{figure}[h]
\begin{center}
\begin{overpic}[width={\ifoc 3.2in \else 3.04in \fi},tics=5]{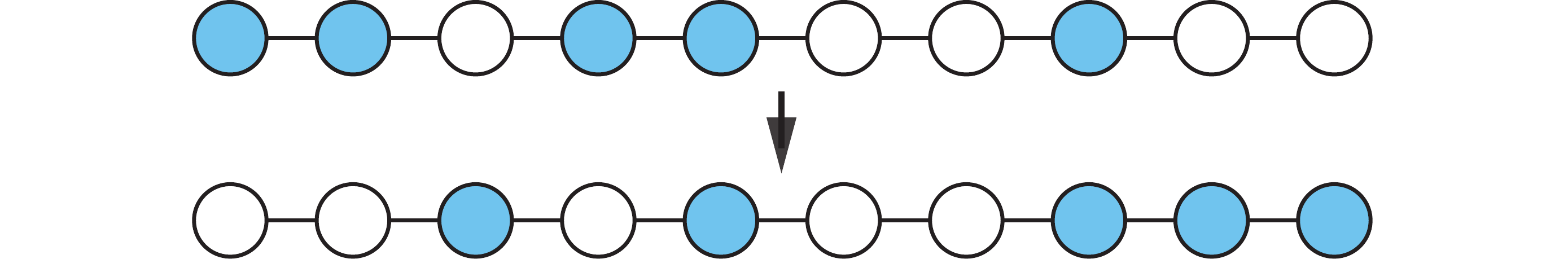}
\end{overpic}
\end{center}
\caption{\label{fig:distribute} Assuming a length $\ell$ path is embedded in a
grid, Lemma~\ref{l:distribute} guarantees that the arbitrary distribution of a 
group of robots can be performed using $O(\ell)$ make span.}
\end{figure}

Lemma~\ref{l:distribute} further implies Lemma~\ref{l:swap}. 
Fig.~\ref{fig:lineswap} illustrates graphically the operation realized 
by Lemma~\ref{l:swap}. 
\begin{lemma}[Lemma 7 in \cite{yu2017constant}]\label{l:swap} On a length $\ell$ path 
embedded in a grid, two groups of robots, equal in number and initially 
located on two disjoint portions of the path, may exchange locations in 
$O(\ell)$ makespan. Multiple such exchanges on vertex disjoint paths can 
be carried out in parallel. 
\end{lemma}
\begin{figure}[h]
\begin{center}
\begin{overpic}[width={\ifoc 3.2in \else 3.04in \fi},tics=5]{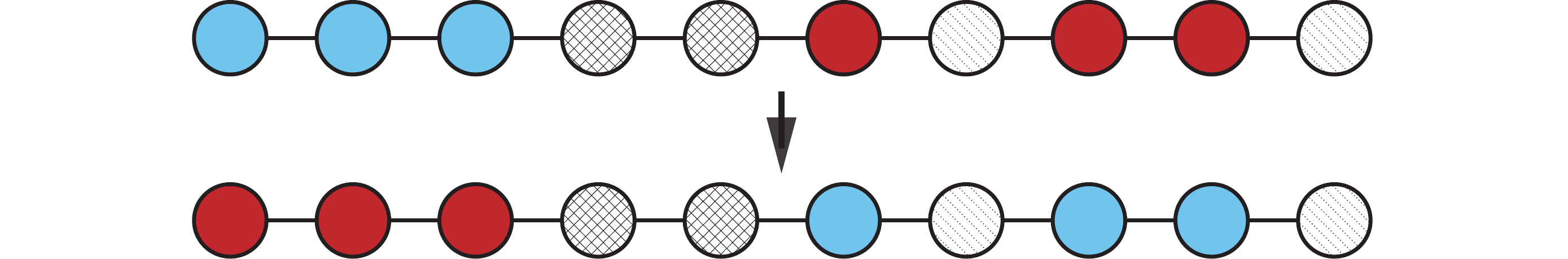}
\end{overpic}
\end{center}
\caption{\label{fig:lineswap} Assuming the grid-embedded path has a length of 
$\ell$, Lemma~\ref{l:swap} guarantees that the swapping of the two separated groups 
of robots, up to $\frac{\ell}{2}$ per group, can be done in $O(\ell)$ make span without
any net movement of other robots on the line.}
\end{figure}

Lemma~\ref{l:distribute} and Lemma~\ref{l:swap} both demand a running time of
$O(\ell^2)$. We note that some problems requires $\Omega(\ell^2)$ time to simply 
write down the solution, e.g., when $\frac{\ell}{2}$ robots need to be moved on 
a path of length $\ell$. Several additional results were developed over 
Lemma~\ref{l:swap} in \cite{yu2017constant} to complete the grouping operation, 
which involves complicated routing of robots on trees, embedded in a grid, that 
may overlap. We provide an alternative method that not only simplifies the process 
with better running time but also allows easy generalization to high dimensions. 
We note that, to complete the grouping operations, using the example from 
Fig.~\ref{fig:split} for illustration, we only need to reconfigure robots on the 
left $5\times 4$ grid so that for each row, robots to be exchanged across the 
split line are equal in number (see Fig.~\ref{fig:permute}). Lemma~\ref{l:swap} 
then takes care of the rest. 

\begin{figure}[h]
\begin{center}
\begin{overpic}[width={\ifoc 3.2in \else 3.04in \fi},tics=5]{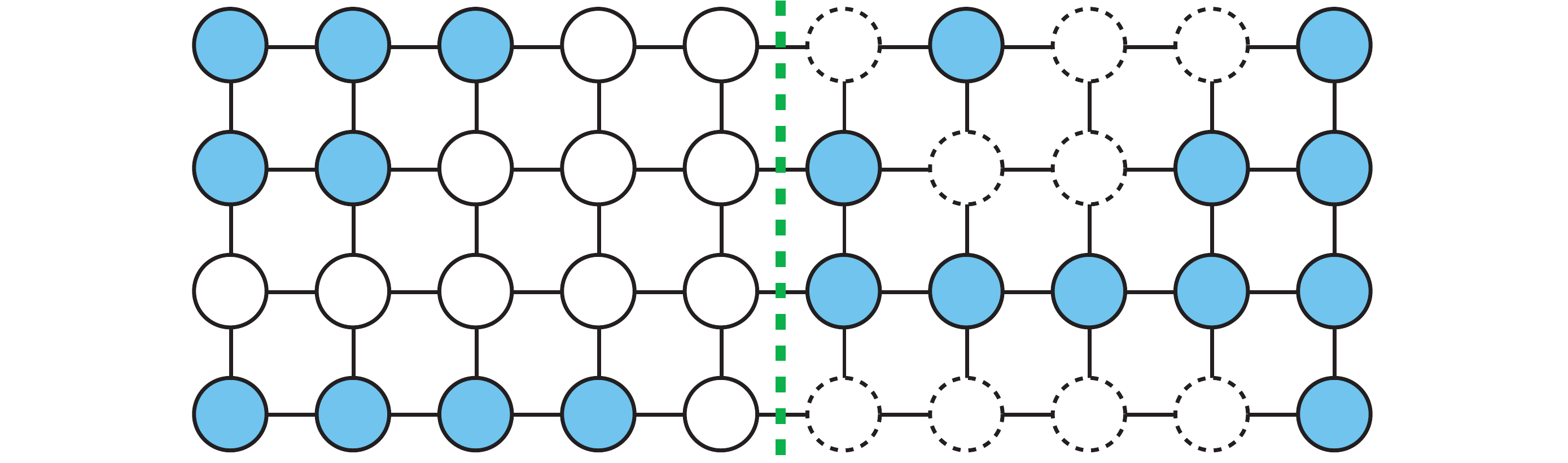}
\end{overpic}
\end{center}
\caption{\label{fig:permute} We would like to reconfigure robots on the left 
$5\times 4$ half of Fig.~\ref{fig:split} to the configuration as shown. The 
right $5\times 4$ portion will not be touched in the operation. In this 
configuration, robots do not need to move between different rows to 
complete the grouping operation, using Lemma~\ref{l:swap}.}
\end{figure}

To perform the reconfiguration, we begin by assigning labels to the robots 
as illustrated in Fig.~\ref{fig:matching-setup} (see the description in the 
figure on how the labels are assigned in a straightforward manner, which takes 
linear time with respect to the size of the grid). These labels are only for 
pairing up robots for the reconfiguration; keep in mind that the shaded robots 
are in fact indistinguishable in the execution of the grouping operation.
\begin{figure}[h]
\begin{center}
\begin{overpic}[width={\ifoc 3.2in \else 3.04in \fi},tics=5]{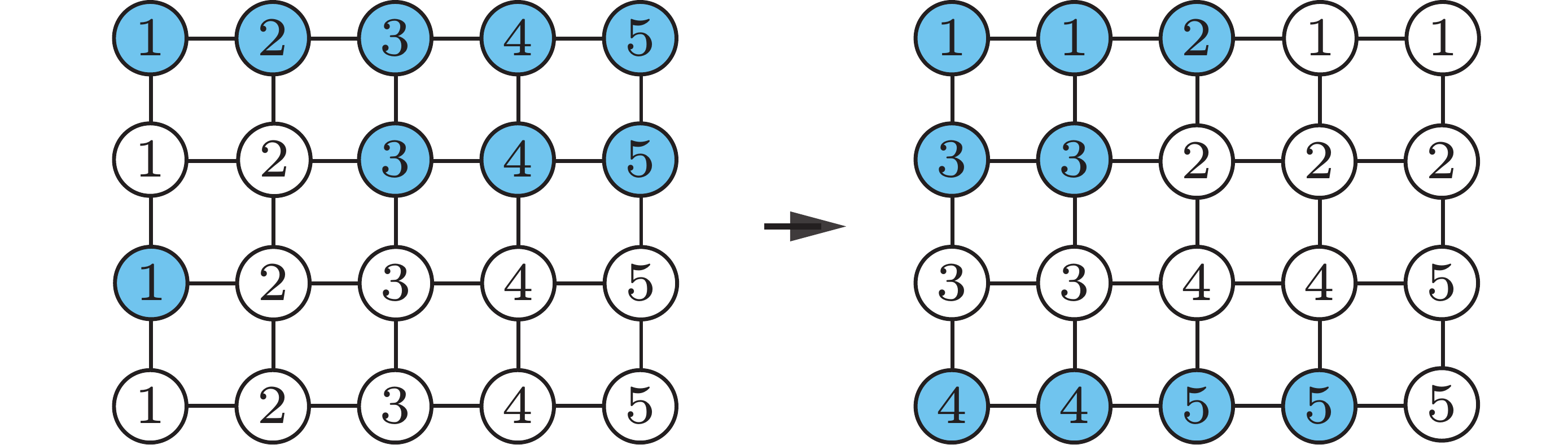}
\put(23, -5){{\small (a)}}
\put(74, -5){{\small (b)}}
\end{overpic}
\end{center}
\caption{\label{fig:matching-setup}(a) and (b) correspond to the left $5\times 4$ 
grids from Fig.~\ref{fig:split} and Fig.~\ref{fig:permute}, respectively. We would 
like to reconfigure the shaded robots to go from (a) to (b) (ignoring the labels). 
In (a), shaded robots are assigned labels based on the column they belong to. In 
(b), from top to bottom and left to right, we sequentially assign each shaded labeled 
robot from (a) a goal. The same is done to the unshaded robots.}
\end{figure}

With the labeling, we set up a bipartite graph as follows. One of the partite set 
$\{v_i^1\}$ (e.g., $\{v_1^1, \ldots, v_5^1\}$ in Fig.~\ref{fig:permute-bipartite}) 
represents the initial columns and the other set $\{v_j^2\}$ (e.g., $\{v_1^2, 
\ldots, v_5^2\}$ in Fig.~\ref{fig:permute-bipartite}) the goal columns. We draw an 
edge between $v_i^1$ and $v_j^2$ if a shaded robot labeled $i$ ends up at a goal
column $j$. For example, in Fig.~\ref{fig:matching-setup}, shaded robots with label 
$1$ in (a) ends up at columns $1$ and  $2$ in (b), yielding the edges $(v_1^1, v_1^2)$ and 
$(v_1^1, v_2^2)$ in Fig.~\ref{fig:permute-bipartite}. If a goal column $j$ contains 
multiple shaded robots with label $i$, then multiple edges between $v_i^1$ and 
$v_j^2$ are added. Note that, if we also add the edges for the unshaded robots in 
Fig.~\ref{fig:matching-setup} in a similar manner, the bipartite graph will be 
$d$-regular where $d$ is the number of rows in the original grid 
($d = 4$ in the provided example). 
\begin{figure}[h]
\vspace*{2mm}
\begin{center}
\begin{overpic}[width={\ifoc 3.6in \else 3in \fi},tics=5]{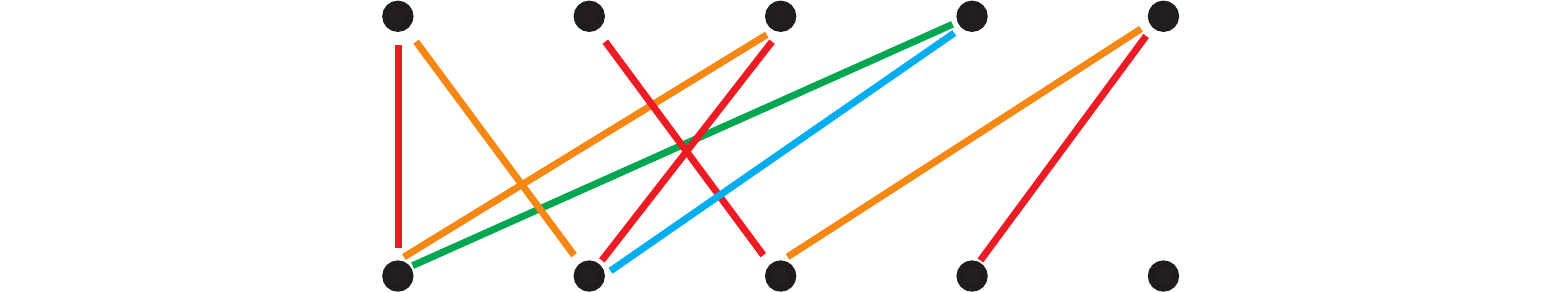}
\put(24,20.5){{\small $v_1^1$}}
\put(35,20.5){{\small $v_2^1$}}
\put(48.5,20.5){{\small $v_3^1$}}
\put(60,20.5){{\small $v_4^1$}}
\put(72,20.5){{\small $v_5^1$}}
\put(24,-4){{\small $v_1^2$}}
\put(35,-4){{\small $v_2^2$}}
\put(48.5,-4){{\small $v_3^2$}}
\put(60,-4){{\small $v_4^2$}}
\put(72,-4){{\small $v_5^2$}}
\end{overpic}
\end{center}
\vspace*{3mm}
\caption{\label{fig:permute-bipartite} A bipartite graph constructed for 
rearranging robots. The $4$ colorings of the edges indicate a possible set of 
$4$ matchings, which are $\{1-1, 2-3, 3-2, 5-4\}$ (red), $\{1-2, 3-1, 5-3\}$ 
(orange), $\{4-1\}$ (green),  $\{4-2\}$ (cyan).}  
\end{figure}
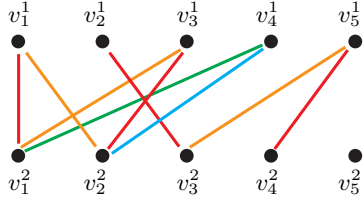

With the bipartite graph constructed, we proceed to obtain a set of up to $d$ 
maximum matchings. We note that this is always possible because our bipartite 
graph is a sub graph of a $d$-regular bipartite graph (By Hall's theorem 
\cite{hall1935representatives}, a perfect matching may be obtained on a 
$d$-regular bipartite graph, the removal of which leaves a $(d-1)$-regular 
bipartite graph). From the obtained set of matchings (e.g., using Hopcroft-Karp 
\cite{hopkroft1973n5}), we permute with Lemma~\ref{l:distribute} to distribute 
the robots vertically so that a robot matched in the $i$-th matching gets moved 
to the $i$-th row. In our example, the first set is $\{1-1, 2-3, 3-2, 5-4\}$, 
which means that a set of three shaded robots labeled $1, 2, 3$, and $5$ 
should be moved to the first row. Doing this for all matching sets shown 
in Fig.~\ref{fig:matching-setup}(a) yields the configuration in 
Fig.~\ref{fig:permute-steps}(a). Then, in a second round, the robots are 
permuted within their row, again using the matching result. In the example, 
the first matching set $\{1-1, 2-3, 3-2, 5-4\}$ says that robots $1, 2, 3$, 
and $5$ on the first row should be moved to columns $1, 3, 2$, and $4$. We 
note that going from $1,2,3,5$ to $1,3,2,4$ is possible with 
Lemma~\ref{l:distribute} because the labels are nominal; we only need to move 
the four indistinguishable robots to columns $1, 2, 3$, and $4$. For the 
configuration in Fig.~\ref{fig:permute-steps}(a), this round yields the 
configuration in Fig.~\ref{fig:permute-steps}(b). We note that the bipartite 
matching technique mentioned here was due to \cite{SzeYu2017}, in which a 
variation of it is used for a different reconfiguration problem.
\begin{figure}[h]
\begin{center}
\begin{overpic}[width={\ifoc 3.2in \else 3.04in \fi},tics=5]{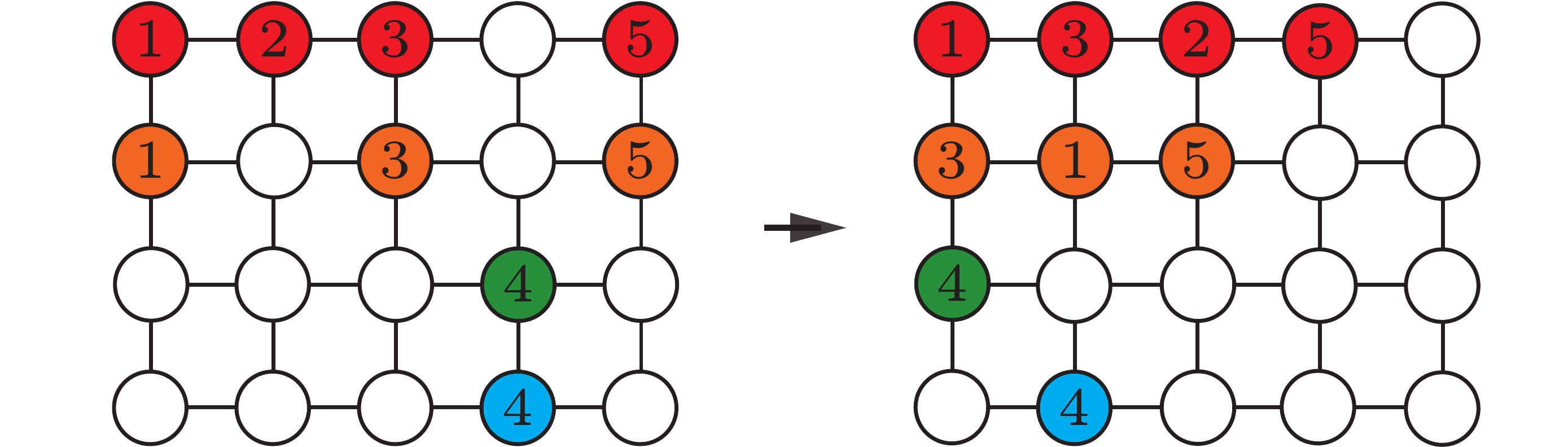}
\put(23, -5){{\small (a)}}
\put(74, -5){{\small (b)}}
\end{overpic}
\end{center}
\caption{\label{fig:permute-steps}(a) The initial permutation of columns of
Fig.~\ref{fig:matching-setup}(a) using the bipartite matching result. (b) A 
second row-based permutation of (a) using the bipartite matching result. Our 
procedure operates following the sequence Fig.~\ref{fig:matching-setup}(a) 
$\to$ Fig.~\ref{fig:permute-steps}(a) $\to$ Fig.~\ref{fig:permute-steps}(b) 
$\to$ Fig.~\ref{fig:matching-setup}(b).}
\end{figure}

We observe that the labeled robots that need to be moved now are all in the 
correct columns. One last column permutation then moves the robots in place. 
In the example, this is going from Fig.~\ref{fig:permute-steps}(b) to 
Fig.~\ref{fig:matching-setup}(b). We summarize the the discussion in the 
following lemma.

\begin{lemma}\label{l:2d-shuffle}On an $m_1 \times m_2$ grid, the reconfiguration 
of a group of indistinguishable robots between two arbitrary configurations can be 
completed using $O(m_1 + m_2)$ makespan in $O(m_1^2m_2 + m_1m_2^2)$ time. 
\end{lemma}
\begin{proof}
The procedure is already fully described; here, we analyze its performance. The 
procedure operates in three phases, each requiring a makespan of either $O(m_1)$ 
or $O(m_2)$ (because only one dimension of the $m_1\times m_2$ grid is involved 
in each phase). The overall makespan is then $O(m_1 + m_2)$. Regarding the 
computation time, each invocation of the procedure from Lemma~\ref{l:distribute} 
or Lemma~\ref{l:swap} on an $m_1 \times m_2$ grid takes $O(m_1^2)$ or $O(m_2^2)$ 
time; doing these in parallel on the grid then takes $O(m_1^2m_2 + m_1m_2^2) = 
O(m_1^2m_2)$ time. For doing the bipartite matching, we may invoke an $O(|E|)$ 
time matching algoithm \cite{cole2001edge} $d$ times to get a $O(d|E|)$ running time where 
$d = m_2$ and $|E_B| = m_1m_2$ are the degree and the number of edges of the 
$d$-regular bipartite graph. The total time spent on matching is $O(m_1m_2^2)$. 
The overall running time is then $O(m_1^2m_2 + m_1m_2^2)$.
\end{proof}

We now generalize Lemma~\ref{l:2d-shuffle} to $k\ge 2$ dimensions. 

\begin{theorem}\label{t:kd-shuffle}On an $m_1 \times \ldots \times m_k$ grid, the 
reconfiguration of a group of indistinguishable robots between two arbitrary 
configurations can be completed using 
\[
O(\sum_{i = 1}^k m_i)
\]
makespan and requires time 
\begin{align}\label{f:reconfigure}
O((\prod_{i=1}^km_i)(\sum_1^km_i)).
\end{align}
\end{theorem}
\begin{proof}
Since the case of $k = 3$ is of practical importance, we first provide the 
proof for this case, which also outlines the inductive proof approach for 
general $k$. On an $m_1 \times m_2 \times m_3$ grid, we partition the gird 
into $m_1m_2$ columns of size $m_3$, in the natural way. To build the 
bipartite graph, robots to be moved will be labeled based on the column it 
belongs to, yielding $m_1m_2$ labels. The goals for these robots are assigned 
sequentially, similar to how it is done in the 2D case. After building the 
bipartite graph as before and performing the matching, the robots to be 
moved are partitioned into $m_3$ {\em layers} (a layer in the 3D case 
corresponds to a {\em row} in the 2D case) with each layer being an $m_1 
\times m_2$ grid.

Then, as in the 2D case, a column permutation is done for each of the $m_1m_2$ 
columns, in parallel. To be able to move the robots on each layer which is a 
$m_1\times m_2$ grid, we invoke Lemma~\ref{l:2d-shuffle} in parallel on all 
$m_3$ layers. This is then followed by a final parallel column permutation. 

To count the makespan, the initial and final column permutations require 
$O(m_3)$ makespan and working with the layers requires $O(m_1 + m_2)$ makespan, 
yielding a total makespan of $O(m_1 + m_2 + m_3)$. For running time, at 
the top layer, the bipartite matching process creates a bipartite graph 
$G_B = (V_B, E_B)$ with $|E_B| = m_1m_2m_3$. The time for doing $d = m_3$ 
matchings is then $O(m_1m_2m_3^2)$. The initial and final column permutation 
takes time $O(m_1m_2m_3^2)$ (because we need to arrange $m_1m_2$ columns 
of size $m_3$ each). For handling the $m_3$ layers of $m_1 \times m_2$ grids, by 
Lemma~\ref{l:2d-shuffle}, it takes time $O(m_1^2m_2m_3 + m_1m_2^2m_3)$. The 
overall running time is then $O(m_1^2m_2m_3 + m_1m_2^2m_3 + m_1m_2m_3^2)$.

For constructing the inductive proof, suppose for dimension $k$, our makespan 
hypothesis for reconfiguration is $O(m_1 + \ldots + m_k)$. The running time 
hypothesis is as given in the theorem statement. 
For dimension $k+1$, the problem is first approached at the top level to generate
$m_{k+1}$ ``layers'' of size $\prod_{i=1}^km_i$ each (corresponding to a 
$m_1\times \ldots \times m_k$ grid). After permuting 
$\prod_{i=1}^km_i$ columns of size $m_{k+1}$, $m_{k+1}$  
$k$-dimensional problems are then solved via the induction hypothesis. 
Lastly, another column permutation is performed to complete the reconfiguration. 

To count the makespan required, we note that at dimension $k$, the initial and 
final column permutations require a makespan of $O(m_{k+1})$ as all 
$\prod_{i=1}^km_i$ columns of size $m_{k+1}$ can be operated on in parallel. By the 
induction hypothesis, the total makespan is then $O(m_1 + \ldots + m_{k+1})$,
which actually does not directly depend on the dimension. The running time 
for the first matching operation takes $O((\prod_{i=1}^{k}m_i)m_{k+1}^2)$ 
time. The running time for the initial and final column permutations require 
calling the $O(m_{k+1}^2)$ routine (Lemma~\ref{l:distribute}) $\prod_{i=1}^km_i$ 
times, taking the same amount of time. By the induction hypothesis, handling the 
$m_{k+1}$ layers take time $m_{k+1}$ multiple of~\eqref{f:reconfigure}. Putting 
these together yields again~\eqref{f:reconfigure} with $k$ replaced by $k+1$. 
\end{proof}

A case of special interest is when all $m_i$, $1 \le i \le k$, are about the same.

\begin{corollary}\label{c:uniform}
On a $k$-dimensional grid with all sides having lengths $O(|V|^{\frac{1}{k}})$, 
the reconfiguration of a group of indistinguishable robots between two arbitrary 
configurations can be completed using $O(k|V|^{\frac{1}{k}})$ makespan and 
$O(|V|^{\frac{k+1}{k}})$ time.
\end{corollary}

Replacing the tree-routing based grouping operation in \sag with the updated, 
staged grouping routine, we obtain the following improved result. 

\begin{theorem}\label{t:isag}Let $(G, X_I, X_G)$ be an \mpp instance
with $G$ being an $m_1 \times \ldots \times m_k$ grid for some $k \ge 2$. 
Then, a solution with 
\[
O(k\sum_{i = 1}^k m_i)
\]
makespan can be computed in time
\begin{align}\label{f:sagmpp}
O(k(\prod_{i=1}^km_i)(\sum_1^km_i)).
\end{align}
\end{theorem}
\begin{proof}
Similar to \sag, standard divide-and-conquer is applied that iteratively divides 
$G$ and subsequent partitions into equal halves; the grouping operation is then 
applied. For the grouping operation, after reconfiguration on a half grid, a parallel
invocation of Lemma~\ref{l:swap} is needed to move the robots  across the splitting
boundary, which takes at most $O(m_1^2m_2\ldots m_k)$ time. Because 
$O(m_1^2m_2\ldots m_k)$ is already a term in~\eqref{f:reconfigure}, this additional 
operation does not contribute to more computation time in an iteration of \sag. 

For a $k$-dimensional grid, in the first $k$ iterations, we may choose the $d$-th round 
to divide dimension $d$ into two halves (i.e., $m_d' = \frac{m_d}{2}$). Following this 
scheme, for the $d$-th round, the makespan is 
\[
O(\frac{m_1}{2} + \ldots + \frac{m_{d-1}}{2} + m_{d} + \ldots + m_k).
\]
For computation time, we need to operate on $2^{d-1}$ subproblems with each subproblem
requiring time no more than 
\[
O(2^{-d+1}(\prod_{i=1}^km_i)(\sum_1^km_i)).
\]

That is, each of the first $k$ iterations takes no more time than~\eqref{f:reconfigure}. 
Tallying up, the first $k$ rounds require makespan and running time as given 
in the theorem statement. 

After $k$ rounds of division, all dimensions are halved. To complete the next
$k$ rounds, the required makespan is halved and the computation time shrinks 
even more (since it's quadratic in at least one of the dimensions and super 
linear in the rest). Subsequently, the makespan and the running time for the 
first $k$ rounds dominate. 
\end{proof}

To distinguish our modification with \sag, we denote the improved \sag algorithm 
as \isag. We mention that \isag runs in quadratic $O(|V|^2)$ time if we allow 
$G$ to be degenerate, i.e., $m_1 = \Omega(|V|)$. To see that this is true, we 
observe that the term inside~\eqref{f:sagmpp} is bounded by 
\[
k^2 m_1\prod_{i = 1}^k m_i = k^2|V|^2 \prod_{i = 1}^{k-1} m_i^{-1} 
< k^22^{-k+1}|V|^2
\]
because $m_i \ge 2$. The last term is $O(|V|^2)$ since $k^22^{-k+1}$ is bounded 
by some small constant. As noted, the quadratic bound is sometimes necessary 
when $G$ is degenerate (see discussion following Lemma~\ref{l:swap}). We note 
that in this case, the running time lower bound can also be $\Omega(|V|^2)$. 
When $G$ is non-degenerate, \isag runs in a sub-quadratic $o(|V|^2)$ time that 
approaches $O(|V|)$.  

We conclude this section with a corollary, which will be useful later, that directly 
follows Corollary~\ref{c:uniform} and Theorem~\ref{t:isag}. 
\begin{corollary}\label{c:sagmpp23}
When all dimensions of the underlying grid are of similar magnitude, the makespan 
and computation time for solving an \mpp instance are $O(\sqrt{|V|})$ and 
$O(|V|^{\frac{3}{2}})$, respectively, for two dimensions. For three dimensions, 
these are $O(|V|^{\frac{1}{3}})$ and $O(|V|^{\frac{4}{3}})$, respectively. For 
general $k$, these are $O(k^2|V|^{\frac{1}{k}})$ and $O(k^2|V|^{\frac{k+1}{k}})$, 
respectively.
\end{corollary}

\section{From Average Case to Worst Case: A Solution Sketch for Two Dimensions}\label{section:sketch}
In this section, we highlight, at a high level, why solution produced 
by \isag can be rather undesirable in practice and how its shortcomings can 
be addressed with the \pafalgo (\paf) algorithm. In sketching \paf, we 
resort to the frequent use of figures to illustrate the important steps. 
We emphasize that the steps explained using these pictorial examples are also 
rigorously proved to be correct later in Section~\ref{section:proofs}. Full 
optimality and running time analysis will also be delayed until then. 

\subsection{The Difficulty}
Given an \mpp instance $p = (G, X_I, X_G)$, let the makespan computed by \isag
be denoted as $d_{\isag}(p)$. From an algorithmic perspective, \isag delivers 
$O(1)$-approximate makespan optimal solutions {\em on average}, i.e., 
for a fixed $G$, let all instances of \mpp on $G$ be $\{p_i = (G, X_I^i, X_G^i)\}$, 
then \isag ensures the quantity (as a sum of ratios) 
\[
\sum_i \frac{d_{\isag}(p_i)}{d_g(p_i)}
\]
is a constant. A key assumption in the average case analysis is that all 
instances $\{p_i\}$ for a fixed $G$ are equally likely, implying a uniform 
distribution of problem instances. When this assumption does not hold, as is 
the case in many practical scenarios, \isag no longer guarantees $O(1)$ 
approximation. Such cases may be illustrated with a simple example. On an $m 
\times m$ grid, let an \mpp instance be constructed so that to reach the goal 
configuration, all robots on the outer boundary must rotate synchronously 
once in the clockwise direction (see Fig.~\ref{fig:sag-non-const}). The 
minimum makespan of the instance is $1$ but \isag will incur a makespan of 
$O(m)$ due to its divide-and-conquer approach that agnostically divide the 
grid in the middle. 

\begin{figure}[h!]
\begin{center}
\begin{overpic}[width={\ifoc 2.4in \else 2.4in \fi},tics=5]{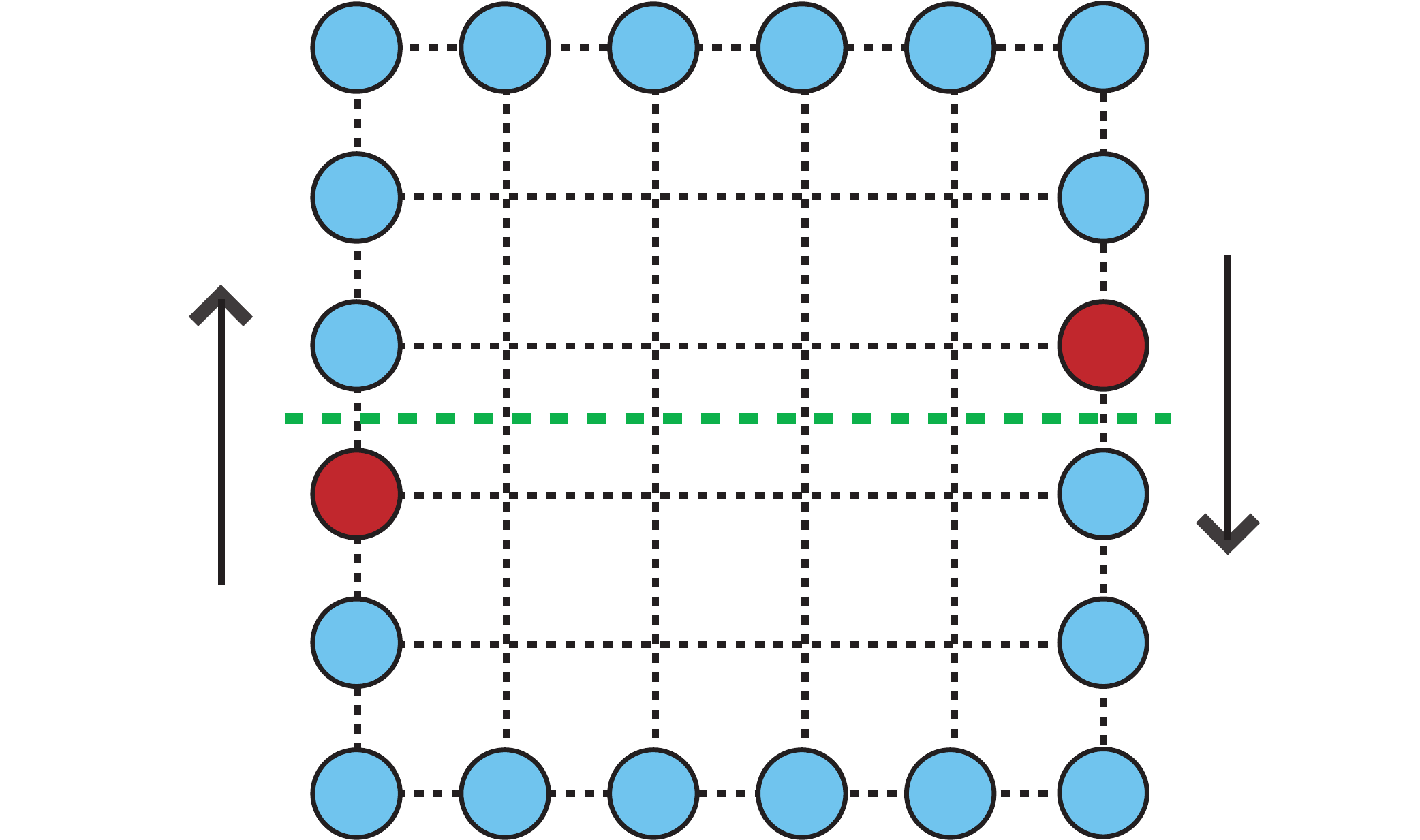}
\end{overpic}
\end{center}
\caption{\label{fig:sag-non-const} An \mpp instance on an $m\times
m$ grid. Solving the instance requires all robots on the outside
perimeter to move clockwise once. \isag will first cause the two (red,
darker shaded) robots to exchange locations, which induces a makespan 
of $O(m)$.}  
\end{figure}

On the other hand, if a polynomial time algorithm can be constructed that 
always produces $O(d_g(p))$ makespan for an arbitrary \mpp instance $p$, 
then $O(1)$-approximate optimal solution can always be guaranteed. 
Naturally, such an algorithm will necessarily require some form of 
divide-and-conquer on top of which the {\em flow} of robots at the 
{\em global} scale must also be dealt with. The key to establishing such 
an algorithm is to be able to recognize the global flow to generate 
appropriate {\em local} routing plans. In terms of the example illustrated 
in Fig.~\ref{fig:sag-non-const}, the two darker shaded (red) robots must 
be routed {\em locally} across the thick dashed (green) boundary lines. This
implies that all the shaded robots must more or less move along synchronously 
around the cycle. A main challenge is how to realize such local-global 
coordination when many such cyclic flows are entangled under maximum 
robot density. 

Here, we mention that the special case of $d_g(p) = 1$ can be easily handled
for an arbitrary dimension $k$. 

\begin{proposition}
Let $G$ be $k$ dimensional grid with $k \ge 2$ and let $p = (G, X_I, X_G)$ 
be an \mpp instance with $d_g(p) = 1$. Then an $O(1)$ makespan plan for 
solving $p$ can be computed in $O(k|V|)$ time. 
\end{proposition}
\begin{proof}
In this case, for a given robot $i$, if $X_I(i) \ne X_G(i)$, its goal is just 
one edge away. Starting from any robot $i$, the vertices $v_1 = X_I(i), v_2 = 
X_G(i), v_3 = X_G(X_I^{-1}(v_2)), v_4 = X_G(X_I^{-1}(v_3)), \ldots$ induce a 
cycle on $G$. When such a cycle has two vertices, this represents an exchange of 
two robots. Using parallel swapping operations, such exchanges can be completed
in $O(1)$ makespan, which leave only simple cycles on $G$ that are all 
disjoint. Robots on these simple cycles can then move to their goals in a 
single synchronous move. The total makespan is then $O(1)$ and to compute the 
plan is to simply write down the cycles, which takes time linear with respect
to the size of the grid. The factor $k$ comes from the search branching factor.    
\end{proof}

\subsection{Sketch of \pafalgo}
In sketching the \paf algorithm, we remark that \paf essentially works on 
a problem $(G, X_I, X_G)$ by gradually updating $X_I$. That is, it first creates 
some intermediate $X_G^1$ based on $X_I$ and $X_G$ and solve the problem 
$(G, X_I, X_G^1)$, leaving a new problem $(G, X_G^1, X_G)$. Then, it repeats 
the process to create and solve another problem $(G, X_G^1, X_G^2)$, resulting 
a new problem $(G, X_G^2, X_G)$. The process continues until $X_I$ is updated 
to eventually match $X_G$. It is important to keep this in mind in reading 
the sketch of \paf. 

In our description of \paf in this section, a two-dimensional, $m_1 \times m_2$ 
grid will be assumed. The generalization to a $k$-dimensional grid will use the 
same general approach but require more involved treatment. As the name suggests, 
\paf partitions an \mpp instance on a grid into small pieces and organize the 
flow of robots through these pieces globally. The partition is essentially a 
form of decoupling that includes and is more general than \isag's half-half 
splitting scheme. 

For a given \mpp instance $p = (G, X_I, X_G)$ with $G$ being an $m_1 \times m_2$
grid, \paf starts by computing $d_g(p)$, the distance gap for the 
problem\footnote{Henceforth, we use $d_g$ in place of $d_g(p)$ because the 
instance is always fixed (but arbitrary); $d_g(p)$ is otherwise only used in 
theorem statements when a problem $p$ is being specified.}. In the main case, 
$d_g = o(m_2)$. That is, for any robot $i$, $d(X_I(i), X_G(i)) = o(m_2)$. This 
means that $G$ may be partitioned into square cells of sizes $5d_g \times 
5d_g$ each. This is the {\em partition} operation in \paf (see 
Fig.~\ref{fig:partition} for an illustration). For the moment, we assume that 
a perfect partition can be achieved, i.e., $m_1$ and $m_2$ are both integer 
multiples of $5d_g$; the assumption is justified in 
Section~\ref{section:proofs}.

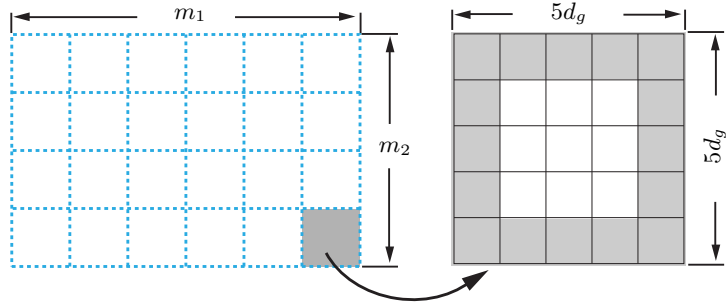
\begin{figure}[h!]
\begin{center}
\begin{overpic}[width={\ifoc 4in \else 3.2in \fi},tics=5]{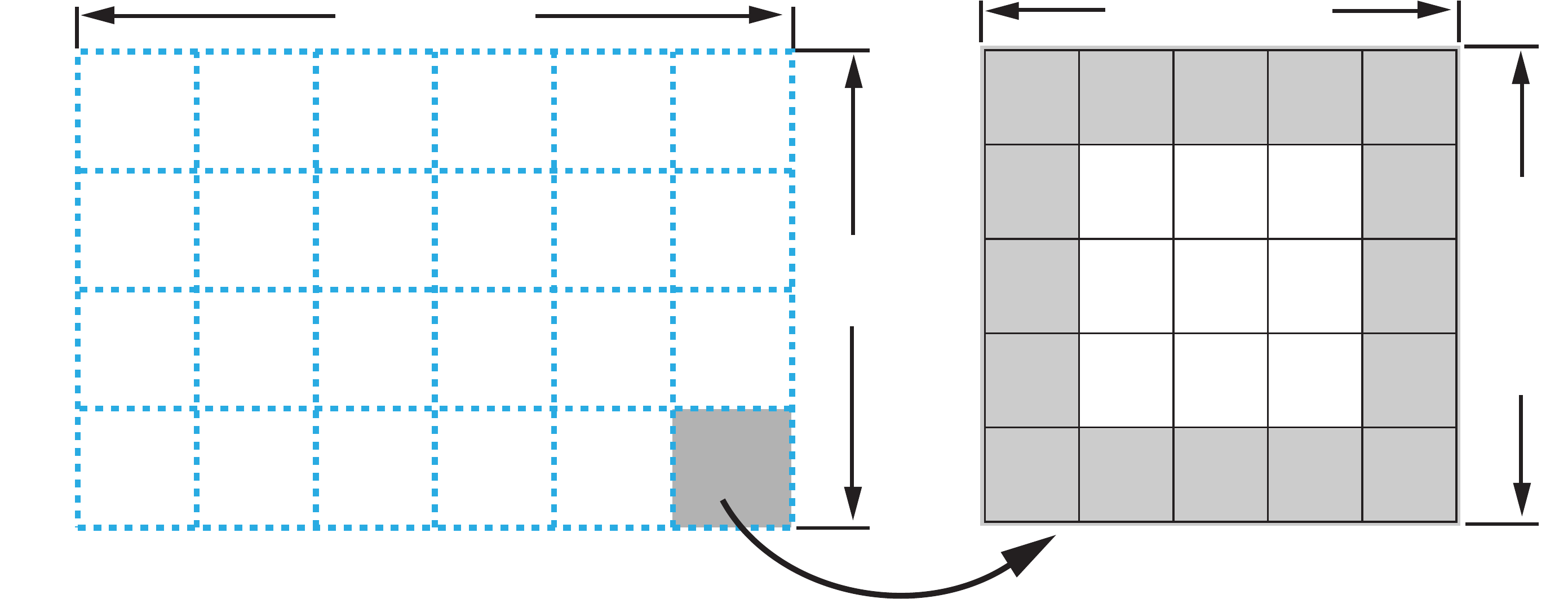}
\put(26.2,37){{\small $m_1$}}
\put(53,19.5){{\small $m_2$}}
\put(75.5,37){{\small $5 d_g$}}
\put(95.5,18){\rotatebox{90}{{\small $5 d_g$}}}
\end{overpic}
\end{center}
\caption{\label{fig:partition} Partitioning of an $m_1 \times m_2$ grid
into $6 \times 4$ cells. Each cell has a size of $5 d_g \times 
5 d_g$. Within a cell (the figure on the right), only robots located 
of a distance no more than $d_g$ from the border may have goals outside 
the cell.}  
\end{figure}

The partition scheme, as a refinement to the splitting scheme from \isag, has 
the property that only robots of distance $d_g$ from a cell boundary may have 
goals outside the cell by the definition of $d_g$ (for more details, see 
Fig.~\ref{fig:boundary}). This means that between two cells that share a 
vertical or horizontal boundary, at most $10 d_g^2$ robots need to cross 
that boundary. If we only count the net exchange, then the number reduces to 
$5 d_g^2$. 

\begin{figure}[h]
\begin{center}
\begin{overpic}[width={\ifoc 2.5in \else 2in \fi},tics=5]{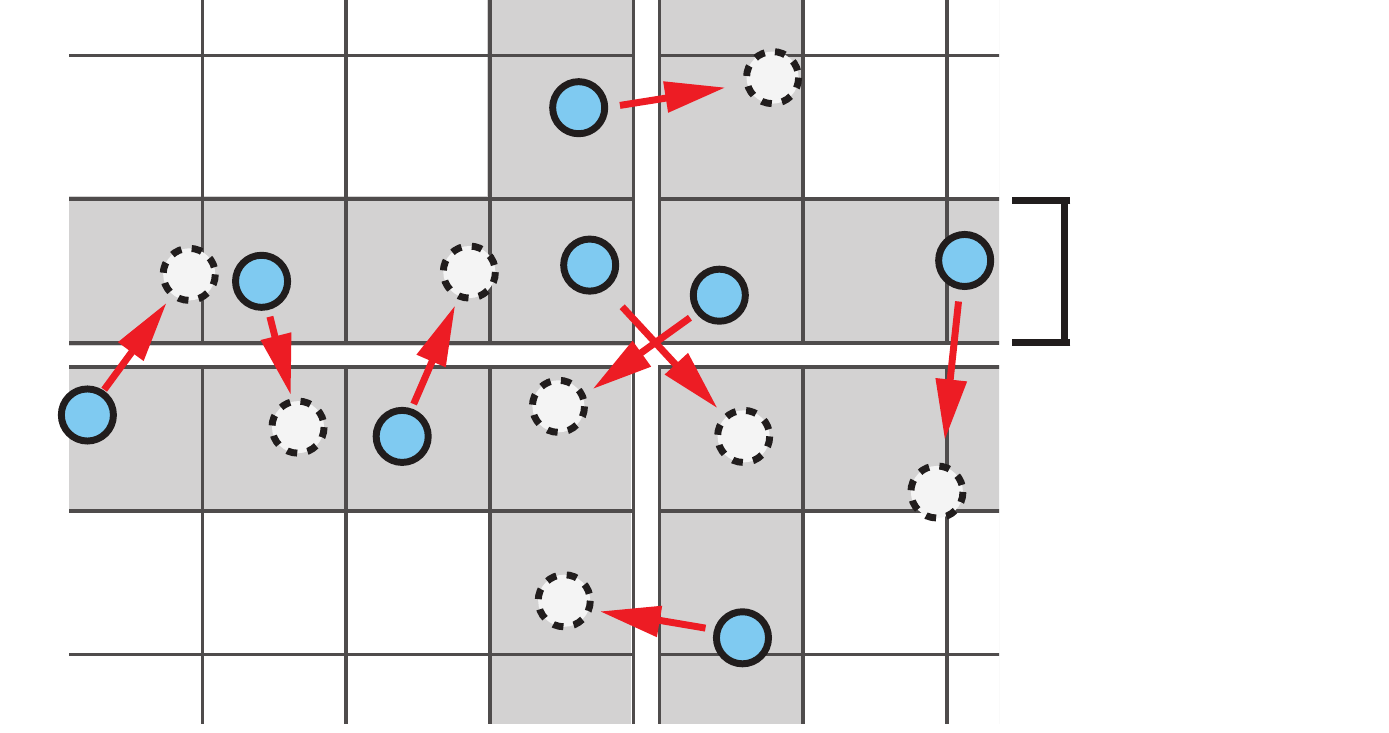}
\put(80,31.5){{\small $d_g$}}
\end{overpic}
\end{center}
\caption{\label{fig:boundary} An illustration of the $d_g$ thick boundary areas 
of four adjacent cells. Any net robot exchange between two cells must happen in this
region by the definition of $d_g$.}  
\end{figure}

Over the partition, \paf will build a flow between the cells treating each cell as 
a node in a graph. To be able to translate the flow into feasible robot movements,
the flow should only happen between adjacent cells that share a boundary. However, 
as illustrated in Fig.~\ref{fig:boundary}, it is possible for a robot to have 
initial and goal configurations that are separated into {\em diagonally adjacent} 
cells which do not share boundaries. To resolve this, we may update the goals
for these robots using robots from another cell that is adjacent to both of the 
involved cells. Fig.~\ref{fig:crossover} illustrates how one such robot can be 
processed. We call this operation {\em diagonal rerouting}, which will create a new 
configuration $X_G^1$ of the robots on $G$. \isag is then invoked to solve
$(G, X_I, X_G^1)$. \isag will do so locally on $4d_g \times 4d_g$ regions that 
span equal parts of four adjacent cells. 
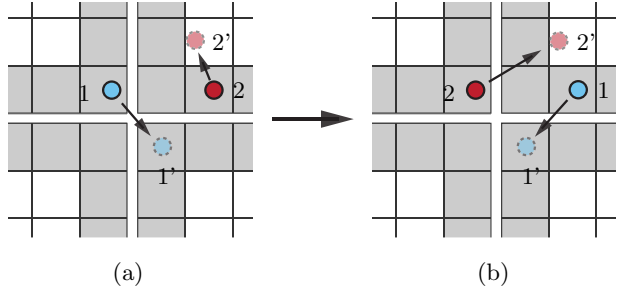
\begin{figure}[h]
\begin{center}
\begin{overpic}[width={\ifoc 3.6in \else 3in \fi},tics=5]{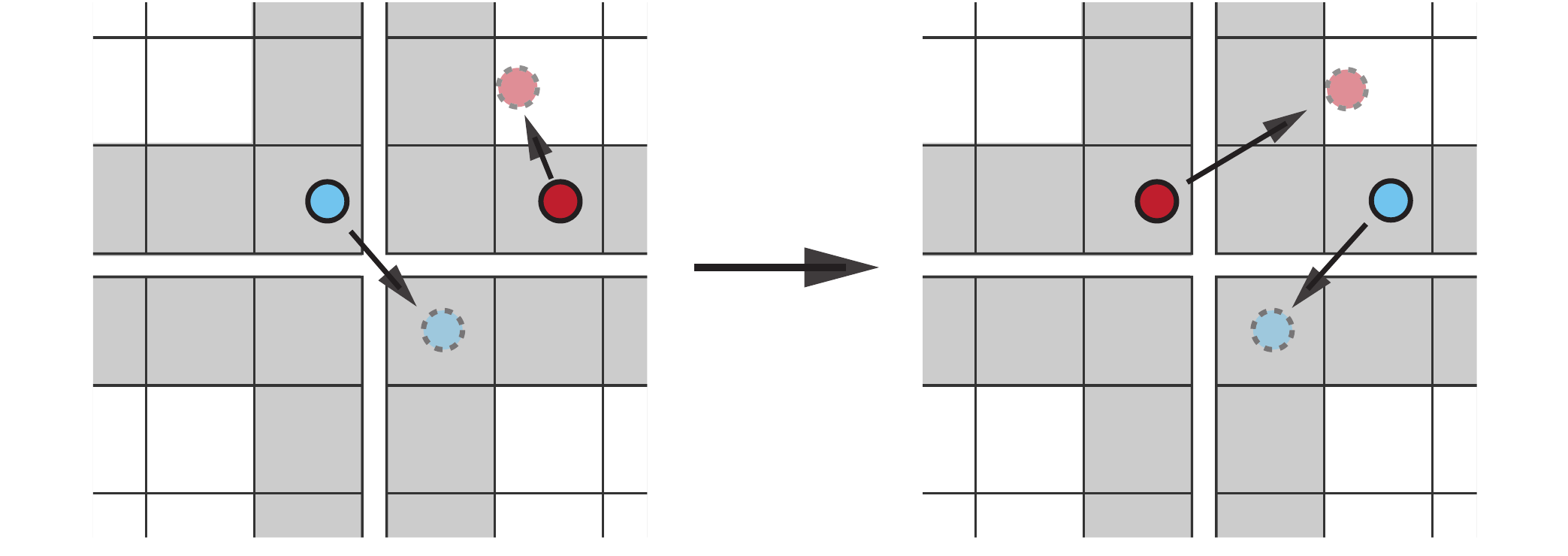}
\put(16, 21){{\small 1}}
\put(27.5, 9){{\small 1'}}
\put(69, 21){{\small 2}}
\put(80.5, 9){{\small 1'}}
\put(35.5, 28.5){{\small 2'}}
\put(38.5, 21.5){{\small 2}}
\put(88.5, 28.5){{\small 2'}}
\put(91.5, 21.5){{\small 1}}
\put(21, -5){{\small (a)}}
\put(74, -5){{\small (b)}}
\end{overpic}
\end{center}
\caption{\label{fig:crossover} (a) At the boundary between four cells, robot $1$
has initial and goal configurations (vertices) spanning two diagonally adjacent 
cells. In the top right cell which is adjacent to both the top left and bottom 
right cells, there exists a robot that has its goal vertex in the same cell. 
(b) By swapping the robots $1$ and $2$ using \isag, no robot needs to cross cell 
boundaries diagonally.}  
\end{figure}

Then, \paf creates another intermediate configuration $X_G^2$ for moving robots 
between each vertical or horizontal cell boundary so that between any two cells, 
robots will only need to move in a single direction when crossing a cell boundary. 
That is, for each cell boundary, \isag is called to ``cancel out'' non-net robot 
movements, as illustrated in Fig.~\ref{fig:cancellation}, leaving only 
uni-directional robot movements across cells. We call this operation {\em flow 
cancellation}.

\begin{figure}[h]
\begin{center}
\begin{overpic}[width={\ifoc 3.6in \else 3in \fi},tics=5]{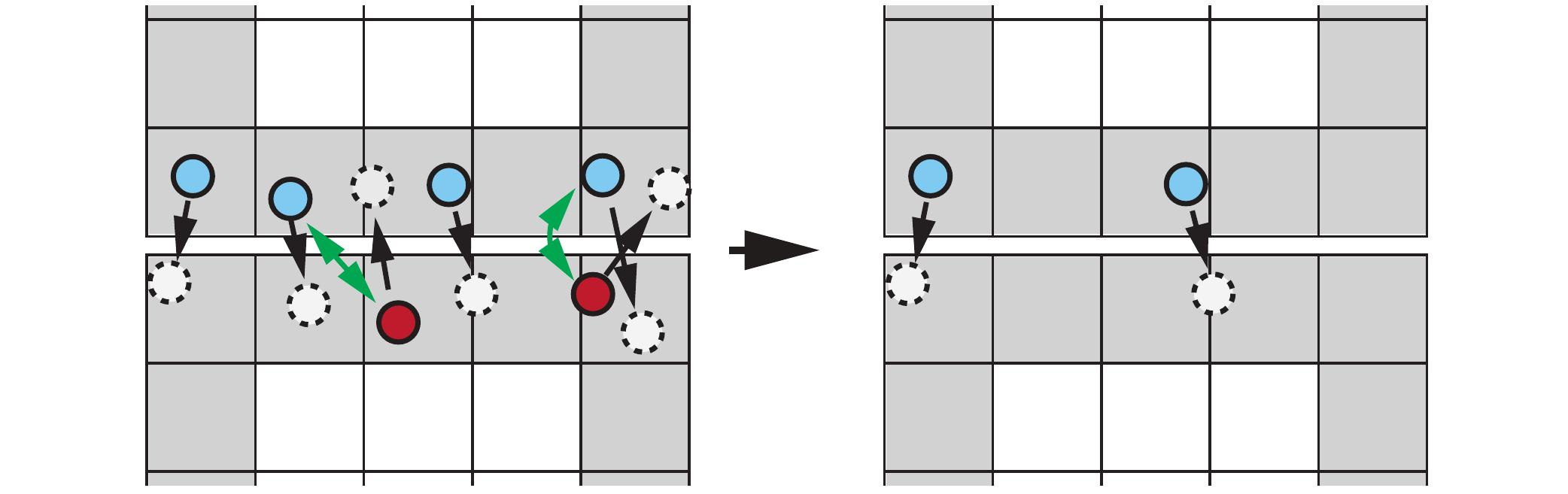}
\put(22, -5){{\small (a)}}
\put(74, -5){{\small (b)}}
\end{overpic}
\end{center}
\caption{\label{fig:cancellation} (a) There are four robots in the top cell and 
two robots in the bottom cell that need to move across the horizontal boundary. 
(b) Through an arbitrary matching (indicated with double sided arrows) of two 
pairs of robots' initial configurations and applying \isag to swap them, the 
robot movements across the boundary are now unidirectional.}  
\end{figure}

The net robot movement across cell boundary induces a flow over the cells (see
Fig.~\ref{fig:flows-decomposition}(a)). Because each cell contains a fixed number 
of robots, the incoming and outgoing flow at each cell (node) must be equal. This
means that all such flows must form a valid {\em circulation}\footnote{A circulation
is essentially a valid flow over a network without source and sink nodes. That is,
the incoming flows and outgoing flows at every node of the network are equal in 
magnitude.} over the graph formed by cells as nodes. The flow between two adjacent 
cells is no more than $6d_g^2$ (to be established later). The circulation can then 
be decomposed into $6d_g^2$ {\em unit circulations} 
(Fig.~\ref{fig:flows-decomposition}(b)). These unit circulations can be translated 
into coordinated {\em global} robot movements that require any robot to travel 
only {\em locally} at most a distance of $O(d_g)$. The translation amounts to 
creating another configuration $X_G^3$. $(G, X_G^2, X_G^3)$ is also solved using 
\isag. 
\begin{figure}[h]
\begin{center}
\begin{overpic}[width={\ifoc 4in \else 3.5in \fi},tics=5]{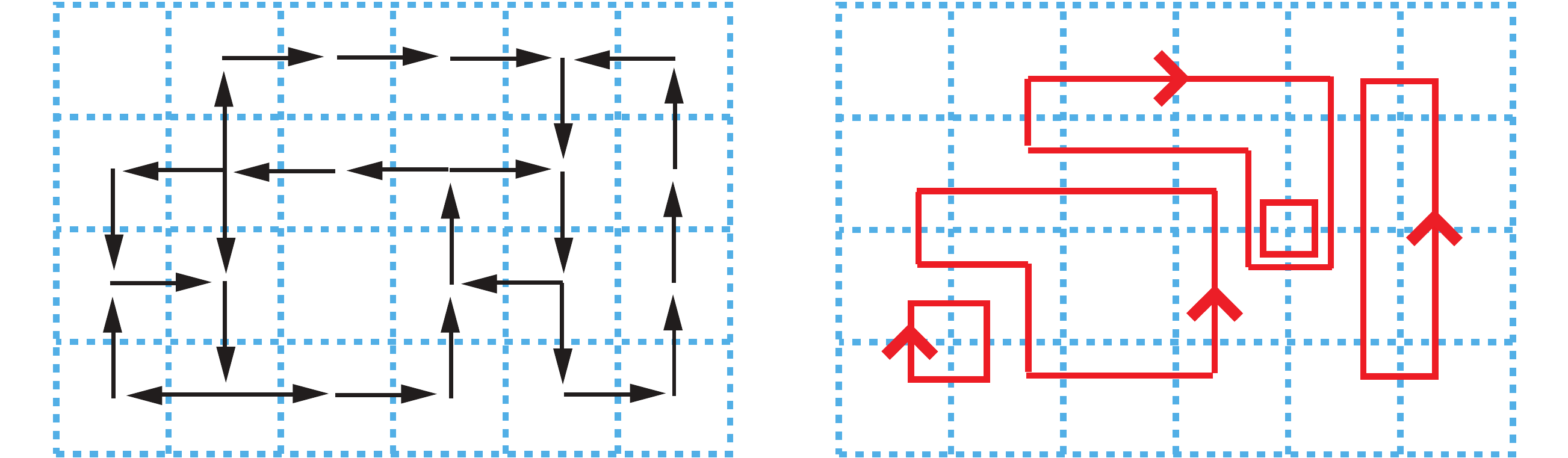}
\put(10,11.8){{\small $2$}}
\put(15.5,7){{\small $2$}}
\put(17.5,19){{\small $2$}}
\put(25,19){{\small $2$}}
\put(29.5,13.5){{\small $3$}}
\put(36.7,13.5){{\small $3$}}
\put(36.7,21.5){{\small $2$}}
\put(32.5,11.8){{\small $2$}}
\put(22.5,-5){{\small (a)}}
\put(72.5,-5){{\small (b)}}
\end{overpic}
\end{center}
\caption{\label{fig:flows-decomposition} (a) Induced circulation (network) from 
required robot movements. The numbers denote the total flow on a given edge. The 
edges without numbers have unit flows. (b) After decomposition, the circulation 
can be turned into unit circulations on simple cycles.}  
\end{figure}
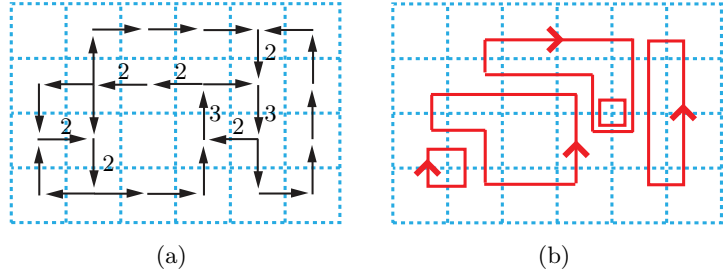

After the preparation phase is done, the scheduled global robot movements can be 
directly executed, yielding a new configuration $X_G^4$. The configuration 
$X_G^4$ has the property that every robot is now in the $5d_g \times 5d_g$ partitioned
cell where its goal resides. \isag can then be invoked to solve $(G, X_G^4, X_G)$ 
(\isag is invoked at the cell level). Throughout the process, each robot only needs 
to move a distance of $O(d_g)$ and calls to \isag can be performed in parallel, 
yielding an overall makespan of $O(d_g)$. Before presenting the details of \paf 
in Section~\ref{section:proofs}, we outline the steps of \paf in Algorithm~\ref{alg:paf}. 
We emphasize that the outline is provided at a very high level that summarizes
the sketch of \paf and only covers the main case in 2D. 

\begin{algorithm}
    \SetKwInOut{Input}{Input}
    \SetKwInOut{Output}{Output}
    \SetKwComment{Comment}{\%}{}
    \Input{$G = (V, E)$: an $m_1 \times m_2$ grid graph \\ 
		$X_I$: initial configuration\\
		$X_G$: goal configuration\\}
    \Output{$M = \langle M_1, M_2, \ldots \rangle$: a sequence of {\em moves} }

\vspace{0.1in}
		
\Comment{\small Partition $G$; $G_S$ represents the partition}
\vspace{0.025in}
$G_S \leftarrow \textsc{Parition}(G, X_I, X_G)$\\
\vspace{0.05in}

\Comment{\small Orienting flows on $G_S$}
\vspace{0.025in}
$M^1, X_G^1 \leftarrow \textsc{DiagoalReroute}(G, G_S, X_I, X_G)$\label{line:reroute}\\
$M^2, X_G^2 \leftarrow \textsc{FlowCancellation}(G, G_S, X_G^1, X_G)$\label{line:cancellation}
\vspace{0.075in}

\Comment{\small Flow decomposition and global route preparation; $P$ are the routes}
\vspace{0.025in}
$M^3, X_G^3, P \leftarrow \textsc{DecomposeFlow}(G, G_S, X_G^2, X_G)$\label{line:prepare}
\vspace{0.075in}

\Comment{\small Global robot routing}
\vspace{0.025in}
$M^4, X_G^4 \leftarrow \textsc{GlobalRouting}(G, P)$\label{line:route}
\vspace{0.075in}

\Comment{\small Final local robot routing}
\vspace{0.025in}
$M^5 \leftarrow \textsc{FinalLocalRoute}(G, G_S, X_G^4, X_G)$\label{line:route}
\vspace{0.075in}

\Return{$M^1 + M^2 + M^3 + M^4 + M^5$}
\caption{\textsc{PafMainCase2D}($G$, $X_I$, $X_G$)} \label{alg:paf}
\end{algorithm}

In closing this section, we note that in providing the details of \paf in 
Section~\ref{section:proofs}, objects of minor importance, including the temporary 
configurations (e.g., $X_G^i$'s) and actual robot movement plans (e.g., $M^i$'s), 
will be omitted in the description. However, sufficient details are provided if a 
reader is interested in deriving these objects. 

\section{\pafalgo in 2D: the Details}\label{section:two-dimensions}
At this point, we make the assumption that for the rest of the paper (unless stated 
explicitly otherwise), for a given problem with $G$ being an $m_1 \times \ldots \times 
m_k$ grid, $d_g = o(m_1)$. Otherwise, $d_g(p) = \Omega(m_1)$ and we may simply
invoke \isag to solve the problem. We note that this is a different condition than 
requiring $G$ being non-degenerate. 
 
We now proceed to provide the full description of how to piece together \paf. The 
goal of this section is to establish the following main result on the existence 
of a polynomial time algorithm (\paf) for computing worst case $O(1)$-approximate 
makespan optimal solution for \mpp, in two dimensions. 

\begin{theorem}\label{t:constant-time-optimal}Let $p = (G, X_I, X_G)$ be 
an arbitrary \mpp instance with $G$ being an $m_1\times m_2$ grid. 
A solution for $p$ with $O(d_g(p))$ makespan can be computed 
in $O(m_1m_2d_g^2)$ deterministic time or $O(m_1m_2d_g + 
m_1m_2\log\frac{m_1m_2}{d_g^2})$ expected time. 
\end{theorem}

Beside the main case 
outlined in Section~\ref{section:sketch}, there is also a special case that needs 
to be analyzed in proving Theorem~\ref{t:constant-time-optimal}, depending on the 
magnitude of $d_g$ relative to $m_1$ and $m_2$. 
The cases for $d_g = o(m_1)$ are divided into 
two disjoint cases: {\em (i)} $d_g = \Omega(m_2)$ and {\em (ii)} $d_g = o(m_2)$. 
The first case can be readily addressed. 

\begin{lemma}\label{l:2dspecial}
Let $p = (G, X_I, X_G)$ be 
an arbitrary \mpp instance in which $G$ is an $m_1\times m_2$ grid with 
$d_g(p) = o(m_1)$ and $d_g(p) = \Omega(m_2)$. The instance admits 
a solution with a makespan of $O(d_g(p))$, computable in $O(m_1m_2d_g(p))$time. 
\end{lemma}
\begin{figure}[h]
\begin{center}
\begin{overpic}[width={\ifoc 4in \else 3.45in \fi},tics=5]{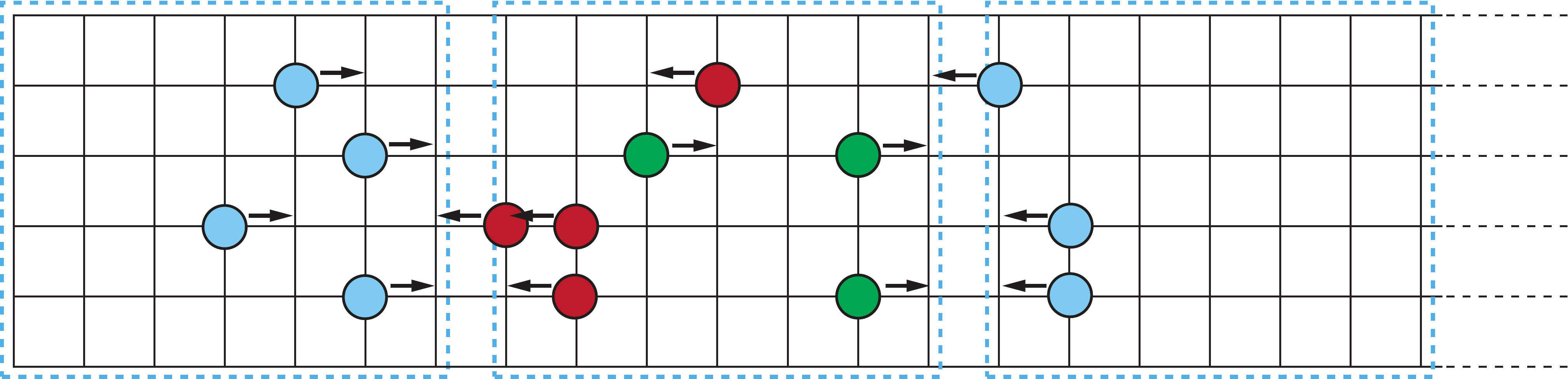}
\put(15, -3){{\small $c_1$}}
\put(46.5, -3){{\small $c_2$}}
\put(78, -3){{\small $c_3$}}
\end{overpic}
\end{center}
\caption{\label{fig:dg-omega-ms} Partitioning of an $m_1 \times m_2$ grid
along the $m_1$ dimension into $q = \lfloor m_1/d_g \rfloor$ cells of roughly 
the same size of $w \times m_2$ with $w \approx m_1/q$. Three
partitioned cells $c_1, c_2$ and $c_3$ are shown. Four robots need to 
move from $c_1$ to $c_2$ and three robots need to move from $c_2$ to $c_3$. Equal
number of robots must move in the opposite direction. The goals of the robots
are not illustrated in the drawing.}  
\end{figure}
\begin{proof}
When $d_g = \Omega(m_2)$, We compute $q = \lfloor m_1/d_g \rfloor$ and $w = 
\lfloor m_1 / q \rfloor$ (note that $w \ge d_g$). Partition $G$ into $q$ 
grid cells along the direction of $m_1$; each cell is of size $m_2 \times w$ 
or $m_2 \times (w + 1)$ (see Fig.~\ref{fig:dg-omega-ms}). Assuming that
$G$ is oriented such that its longer dimension is aligned horizontally, 
then from left to right, we label these cells $c_1, \ldots, c_q$. By the 
definition of $d_g$, a robot initially located in cell $c_i$ may only have 
its goal in either $c_{i-1}$, $c_i$, or $c_{i+1}$ (for applicable $i-1$, and 
$i+1$). This further implies that for any applicable $i$, the number of 
robots that needs to move from $c_i$ to $c_{i+1}$ is the same as the number 
of robots that needs to move from $c_{i+1}$ to $c_i$. The \mpp instance can 
then be solved in two rounds through first invoking \isag on the combined 
cells $c_i+c_{i+1}$ for all applicable odd $i$. This round finishes all robot 
exchanges between $c_i$ and $c_{i+1}$ for odd $i$. In the second round, 
\isag is invoked again to do the same, now for all applicable even $i$. 
Since both parallel applications of \isag incur a makespan of $O(w + d_g) = 
O(d_g)$, the total makespan is $O(d_g)$. For running time, each round of 
\isag application requires $O(q(d_g^2m_2 + d_gm_2^2))$. 
Since $q = O(\frac{m_1}{d_g})$ and $d_g = \Omega(m_2)$, this yields a total 
time of $O(m_1m_2d_g)$.
\end{proof}

The rest of this section is devoted to the case $d_g = o(m_2)$. 
Because $d_g = o(m_2)$, without loss of generality, we assume that $m_1 \ge 
m_2 \ge 5d_g$. Furthermore, we may assume without loss of generality that 
$m_1$ and $m_2$ are multiples of $5d_g$. If that is not the case, assuming 
that \paf is correct, then we can apply \paf up to four times without adding 
makespan or running time penalty. To execute this, first we compute $q_1 = 
\lfloor m_1 / (5d_g) \rfloor$ and $q_2 = \lfloor m_2/(5d_g) \rfloor$. We note 
that $|V| \approx q_1q_2d_g^2$. Then, \paf is applied to the top left portion 
of $G$. This will fully solve the problem for the top left $(q_1 - 1) \times 
(q_2 -1)$ cells of sizes $5d_g \times 5d_g$. Doing the same three more times 
with each application on a different section of $G$, as illustrated in 
Fig.~\ref{fig:multiple}, the entire problem is then solved. 
\begin{figure}[h]
\begin{center}
\begin{overpic}[width={\ifoc 4.5in \else 3.5in \fi},tics=5]{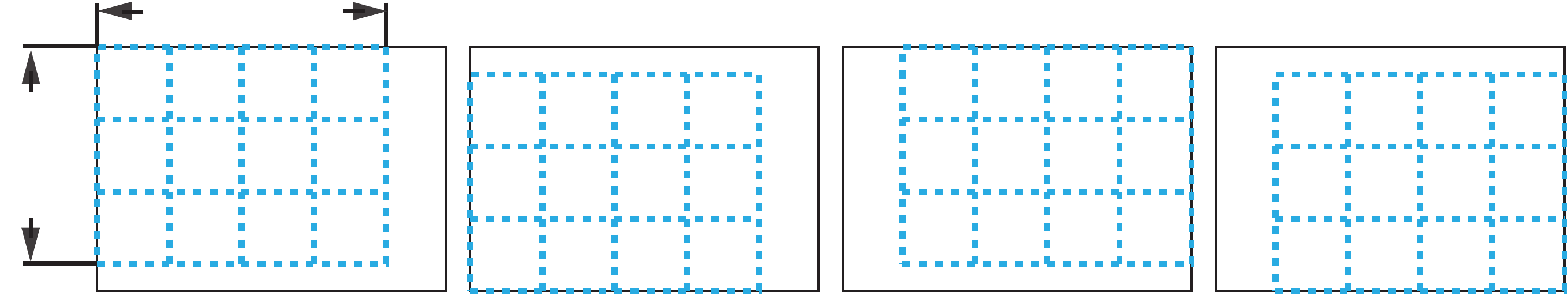}
\put(12,18){{\small $5q_1d_g$}}
\put(0.5,5){\rotatebox{90}{{\small $5q_2d_g$}}}
\end{overpic}
\end{center}
\caption{\label{fig:multiple} For $G = m_1 \times m_2$, if $m_1$ or
$m_2$ are not multiples of $5d_g$, we may apply \paf to a $q_1\times q_2$
cell partition of $G$ up to four times to cover $G$.}  
\end{figure}

Henceforth in this section, we assume $m_1 = 5q_1d_g$ and $m_2 = 5q_2d_g$ in 
which $q_1$ and $q_2$ are integers. $G$ is partitioned into a $q_1 \times q_2$ 
{\em skeleton grid} $G_S$ with its nodes being $5d_g \times 5d_g$ {\em cells}. 
%In graphical illustration, we use $q_1 = 4$ and $q_2 = 3$. 
We remind the readers that after the partition, by the definition of $d_g$, 
robot exchanges between cells can only happen in a $d_g$ wide border for any 
cell, as explained earlier and illustrated in Fig.~\ref{fig:boundary}.
Our immediate goal is to make sure that between neighboring cells, the 
movement of robots are uni-directional and does not happen between diagonally 
adjacent cells. That is, we would like to realize what is illustrated in 
Fig.~\ref{fig:flows-decomposition}(a) from a raw partition, in polynomial 
time and $O(d_g)$ makespan, using {\em diagonal rerouting} 
(Fig.~\ref{fig:crossover}) and {\em flow cancellation} 
(Fig.~\ref{fig:cancellation}) operations.

\begin{lemma}[Flow Orientation]\label{l:orientation}In $O(m_1m_2d_g)$
time and $O(d_g)$ makespan, the flow of robots on the $q_1 \times q_2$ 
skeleton grid may be arranged to be only vertical or horizontal between 
adjacent cells and uni-directional. The largest total incoming flow through 
a cell boundary is no more than $6d_g^2$. 
\end{lemma}
\begin{proof}We first show how to carry out the diagonal rerouting operation. 
For convenience and with more details, we reproduce Fig.~\ref{fig:crossover} in 
Fig.~\ref{fig:crossover-detail} and let the four involved cells be $c_1$ 
through $c_4$ as illustrated. By the definition of $d_g$, if a robot $1$ in $c_1$ 
has its goal in $c_3$, then the robot must be in the bottom right $d_g \times
d_g$ region of $ c_1$ and its goal must be in the top left $d_g \times d_g$ 
region of $c_3$. For each such robot, we pick an arbitrary robot $2$ from $c_2$ 
in the diagonal-line shaded region. Any robot in this region will have its goal in 
$c_2$ (by definition of $d_g$). If we swap the initial configurations of $1$ and 
$2$, then the diagonal movement of $1$ is eliminated. Going in a clockwise 
fashion, for any robot in $c_2$ that needs to move to $c_4$, we can swap it with a 
robot from $c_3$ in the diagonal-line shaded region. Within the $4d_g \times
4 d_g$ region, we create a (temporary) \mpp problem containing only these swaps.
For each such meetings of four cells, such an \mpp instance is created. Then, 
all these disjoint instances can be solved with \isag in parallel using only 
$O(d_g)$ makespan. For computation time, constructing the instance requires a 
single linear scan of the $4d_g\times 4d_g$ region and solving each 
\mpp instance takes $(16d_g^2)^{\frac{3}{2}} = O(d_g^3)$ time, 
by Corollary~\ref{c:sagmpp23}. There are $q_1q_2$ such instances, demanding 
a total time of $O(q_1q_2d_g^3) = O(m_1m_2d_g)$.

\begin{figure}[h]
\begin{center}
\begin{overpic}[width={\ifoc 3.6in \else 3in \fi},tics=5]{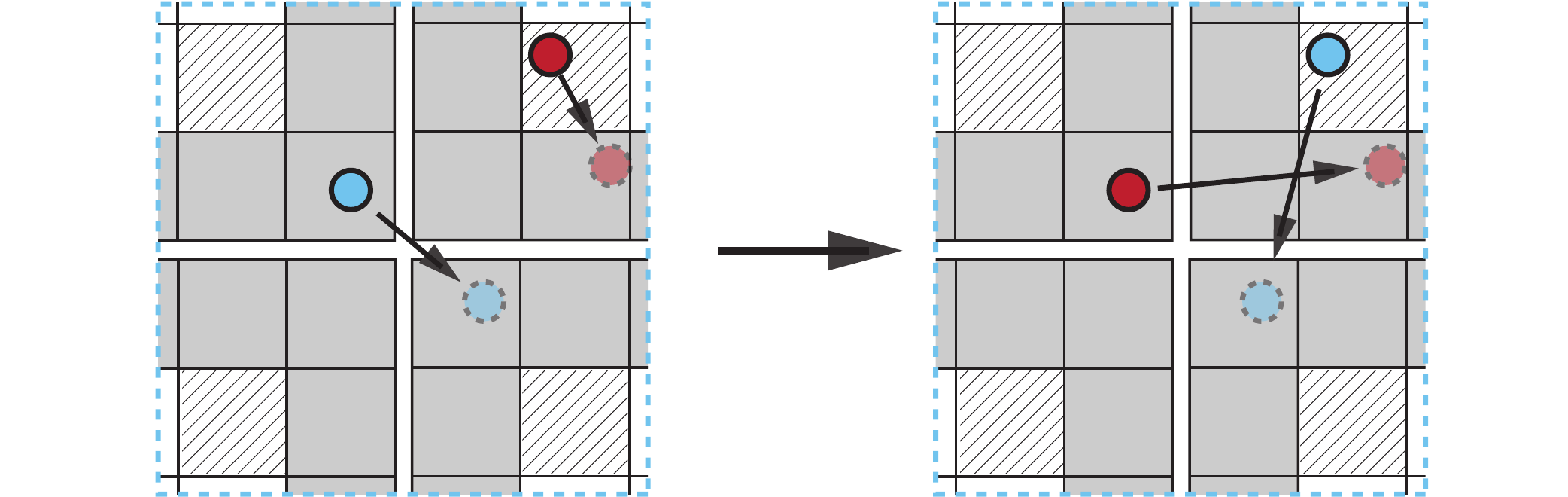}
\put(0,6.7){{\small $c_4$}}
\put(0,22){{\small $c_1$}}
\put(44,6.5){{\small $c_3$}}
\put(44,22){{\small $c_2$}}
\put(18.4,18.3){{\small $1$}}
\put(30,27){{\small $2$}}
\end{overpic}
\end{center}
\caption{\label{fig:crossover-detail} Illustration of four cells meeting at 
corners. Each small square region is of size $d_g \times d_g$. The entire region
is of size $4d_g \times 4 d_g$. Swapping $1$ and $2$ eliminates the need for 
$1$ to directly cross into a diagonally adjacent cell.}  
\end{figure}
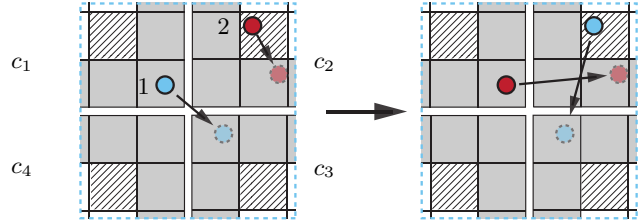

The flow cancellation operation is carried out using a mechanism similar to that
for diagonal rerouting. Referring to Fig.~\ref{fig:cancellation-detail} as an
updated version of Fig.~\ref{fig:cancellation}(a), for a horizontal boundary 
between two adjacent cells $c_1$ and $c_2$, there can be robots that are more than 
$d_g$ away from the boundary that need to cross the boundary. This is due to 
the diagonal rerouting step. Suppose that there are $n_1$ robots that 
need to move from $c_1$ to $c_2$ and $n_2$ from $c_2$ to $c_1$. We may pick 
$\min\{n_1, n_2\}$ robots from each group and create an \mpp problem that swap them 
on the $5d_g \times 4d_g$ region as shown in Fig.~\ref{fig:cancellation-detail}. 
\begin{figure}[h]
\begin{center}
\begin{overpic}[width={\ifoc 3.6in \else 3in \fi},tics=5]{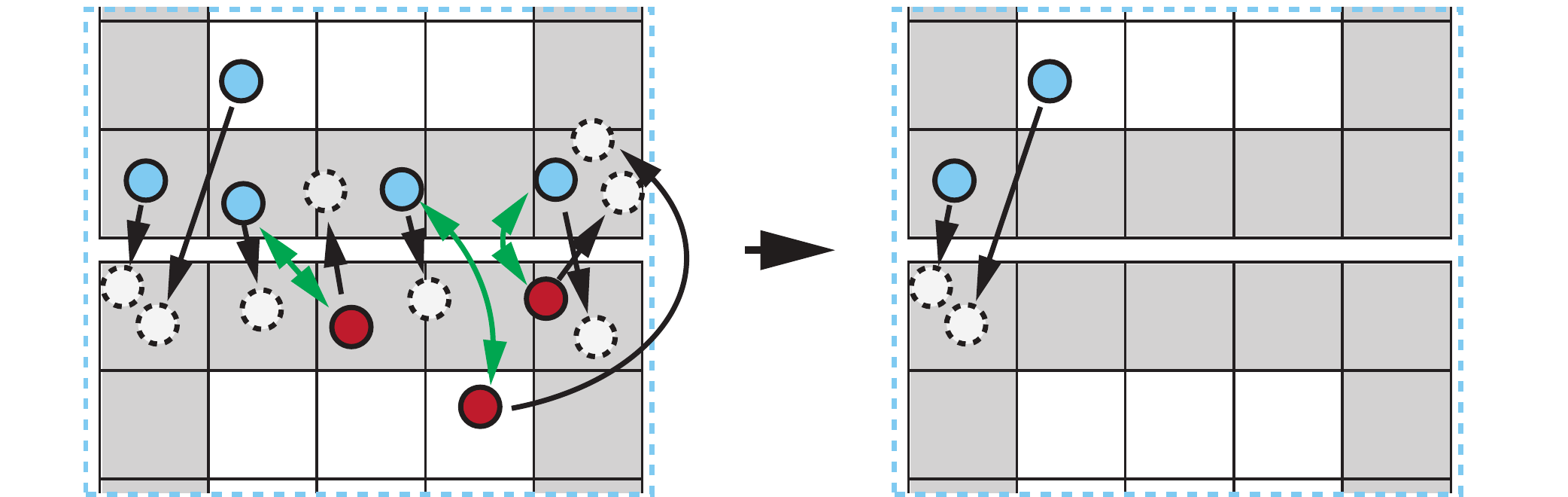}
\put(-2,7.5){{\small $c_2$}}
\put(-2,23){{\small $c_1$}}
\end{overpic}
\end{center}
\caption{\label{fig:cancellation-detail} A horizontal boundary between two adjacent
cells. Some potential robot movements across the boundary are
illustrated. Among these, three pairs of robots, as indicated with double sided 
arrows, may be matched to make the flow  across the boundary uni-directional.}  
\end{figure}
Applying \isag on the instance then renders the flow of robots between the boundary 
uni-directional. By applying \isag in parallel on all such instances over horizontal 
boundaries and then another round over vertical boundaries, flows of robots between 
cell boundaries are all uni-directional. 
Following the analysis of diagonal rerouting step, the flow cancellation operation 
also induces $O(d_g)$ makespan because each \mpp instance is on an $O(d_g)\times 
O(d_g)$ grid region. The running time is also the same as the diagonal rerouting step 
at $O(m_1m_2d_g)$. It is clear that the total flow through any boundary is no more 
than $5d_g^2 + d_g^2/2 + d_g^2/2 = 6d_g^2$.
\end{proof}

After the flow cancellation operation, we are left with only unidirectional flows
on the skeleton grid $G_S$ that are either vertical or horizontal between adjacent 
cells. To route the robots these flows represent, closed disjoint cycles must be 
constructed for moving the robots synchronously across multiple cell boundaries. 
To achieve this, we will first decompose the flow into unit circulations 
(i.e., describing a procedure for going from Fig.~\ref{fig:flows-decomposition}(a) 
to Fig.~\ref{fig:flows-decomposition}(b)). Then, we will show how the cycles on 
the skeleton grid $G_S$ (e.g., Fig.~\ref{fig:flows-decomposition}(b)) can be grouped
into a constant number of $d_g^2$ sized batches and turned into actual cycles on 
the original grid $G$. Our flow decomposition result, outlined below, works for 
arbitrary graphs. 

\begin{theorem}[Circulation Decomposition]\label{t:fd}
Let $\mathcal C$ be a circulation on a graph $G = (V, E)$ with the largest total 
incoming flow for any vertex being $f > 0$. $\mathcal C$ can be decomposed into $f$ 
unit circulations on $G$ in $O(f^2|V|)$ time or $O(f|V|\log|V|)$ 
expected time.
\end{theorem}
\begin{proof}
We proceed to build a bipartite graph over two copies of $|V|$. For a vertex
$v_i \in V$, we denote one of the copy $v_i^1$ (belonging to the first partite set) 
and the other $v_i^2$ (belonging to the second partite set). For any two 
adjacent vertices $v_i, v_j \in V$, if there is a flow of magnitude $f_{ij}$
from $v_i$ to $v_j$, then we add $f_{ij}$ edges between $v_i^1$ and $v_j^2$. 
Because the largest total incoming flow to any vertex is $f$, the maximum 
degree for any $v_i^j$, $j = 1, 2$, is also $f$. Also, due to flow conservation 
at vertices, for fixed $v_i \in V$, $v_i^1$ and $v_i^2$ have the same degree 
$f_i \le f$. For all $v_i$ with $f_i < f$, we add $f - f_i$ edges between 
$v_i^1$ and $v_i^2$. This brings the degrees of all vertices in the bipartite
graph to $f$, yielding a regular bipartite graph. The bipartite graph has $2|V|$
vertices and $f|V|$ edges. An illustration of the bipartite graph construction 
is given in Fig.~\ref{fig:flow-bipartite}.
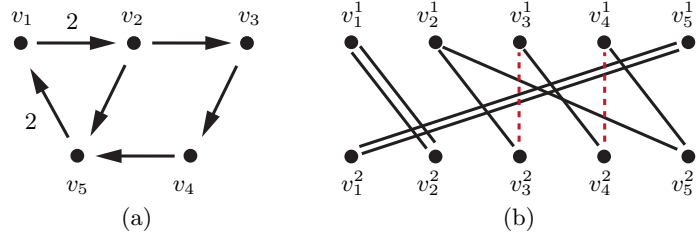
\begin{figure}[h]
\vspace*{5mm}
\begin{center}
\begin{overpic}[width={\ifoc 3.6in \else 3in \fi},tics=5]{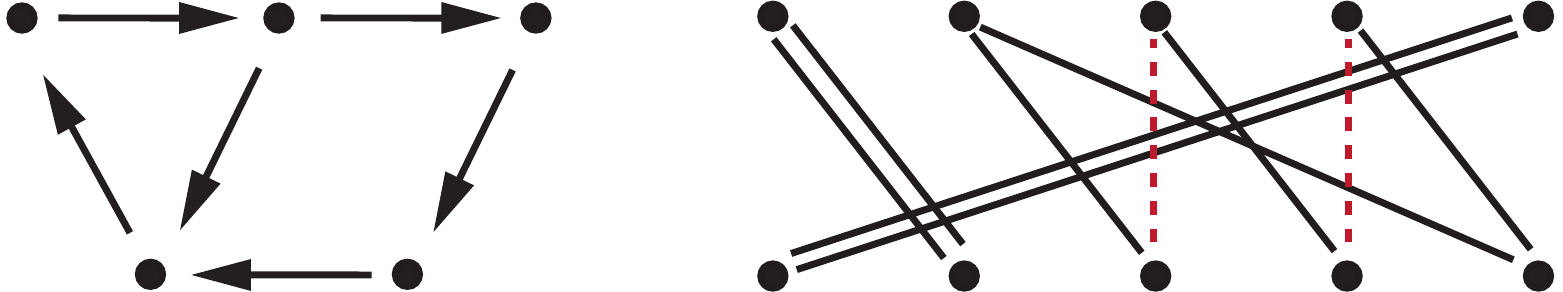}
\put(0,20.5){{\small $v_1$}}
\put(16,20.5){{\small $v_2$}}
\put(33,20.5){{\small $v_3$}}
\put(8,-4){{\small $v_5$}}
\put(8,19.5){{\small $2$}}
\put(2,5){{\small $2$}}
\put(23.5, -4){{\small $v_4$}}
\put(48,20.5){{\small $v_1^1$}}
\put(59,20.5){{\small $v_2^1$}}
\put(72.5,20.5){{\small $v_3^1$}}
\put(84,20.5){{\small $v_4^1$}}
\put(96,20.5){{\small $v_5^1$}}
\put(48,-4){{\small $v_1^2$}}
\put(59,-4){{\small $v_2^2$}}
\put(72.5,-4){{\small $v_3^2$}}
\put(84,-4){{\small $v_4^2$}}
\put(96,-4){{\small $v_5^2$}}
\put(16,-9){{\small (a)}}
\put(71.5,-9){{\small (b)}}
\end{overpic}
\end{center}
\vspace*{3mm}
\caption{\label{fig:flow-bipartite} (a) A graph with five vertices and a valid 
circulation of largest total incoming degree being $2$. The flow on each edge 
with non-unit flow is marked on the edge. (b) The constructed bipartite graph.
The dashed edges are the edges added to make the graph regular.}  
\end{figure}

With the regular bipartite graph of degree $f$, by Hall's theorem \cite{hall1935representatives}, 
there exists a perfect matching that can be computed in $O(|E|) = O(f|V|)$ time \cite{cole2001edge}. 
The matching corresponds 
to a unit circulation on $G$, which translates to either a single cycle or 
multiple vertex disjoint cycles. In the example, a
perfect matching may be $(v_1^1, v_2^2), (v_2^1, v_5^2), (v_3^1, v_3^2), 
(v_4^1, v_4^2), (v_5^2, v_1^2)$, which translates to the cycle $v_1v_2v_5$. 
An application of the perfect matching algorithm reduces the degree
of the bipartite graph by $1$, resulting in another regular bipartite graph. 
We may repeat the procedure $f$ times to obtain $f$ unit circulations 
on $G$. The total running time to obtain the $f$ unit circulations is $O(f^2|V|)$.
Alternatively, we may use the randomized $O(|V|\log |V|)$ perfect matching 
algorithm \cite{goel2013perfect}, which yields a total expected running time of $O(f|V|\log |V|)
= \tilde O(f|V|)$.
\end{proof}

For our setting, Theorem~\ref{t:fd} implies the following (note that $|V| = O(q_1q_2)$
and $f = O(d_g^2)$.) 

\begin{corollary}[Flow Decomposition on Skeleton Grid]\label{c:fd}
An $O(d_g^2)$ circulation on a $q_1\times q_2$ skeleton grid can be decomposed 
into $O(d_g^2)$ unit circulations in $O(m_1m_2d_g^2)$ time or 
$O(m_1m_2\log\frac{m_1m_2}{d_g^2})$ expected time. 
\end{corollary}

Because at most $6d_g^2$ flows can pass through a cell boundary, at most 
$12d_g^2$ flow can pass through a cell (two incoming, two outgoing). 
Corollary~\ref{c:fd} gives us $12d_g^2$ unit circulations over the skeleton 
grid $G_S$. With the decomposed circulation, we may group them into batches 
and translate these into actual robot movements on $G$. To start, we handle a 
$d_g$ batch. 

\begin{lemma}[Single Batch Global Flow Routing]\label{l:gfrsb}
A batch of up to $d_g$ unit circulations on the $q_1 \times q_2$ skeleton 
grid may be translated into actual cyclic paths for robots on $G$ to complete in a 
single step, using $O(m_1m_2)$ time. 
\end{lemma}
\begin{proof}
For a fixed cell, there are many possible orientations for the incoming and 
outgoing flows. However, we only need to analyze the case where all four boundaries 
of a cell have flows passing through. If we can handle these, other cases are 
degenerate ones with some flows crossing the boundaries being zero. Among all 
possible flows that go through all sides of a cell, there are only three possible 
orientations for the incoming and outgoing flows after considering rotation symmetries 
and flow direction symmetries, as illustrated in Fig.~\ref{fig:cell-flow}. For example, 
the case with one incoming flow and three outgoing flows is the same as reversing the 
directions of the arrows in the case shown in Fig.~\ref{fig:cell-flow}(a). Therefore, 
establishing how a $d_g$ amount of flow may be translated into feasible robot movements 
for the three cases in Fig.~\ref{fig:cell-flow} encompasses all possible scenarios. 
We will establish how up to $d_g$ robots can be arranged to go through the boundaries in 
a single step for all three cases.
\begin{figure}[h]
\begin{center}
\begin{overpic}[width={\ifoc 3.6in \else 3in \fi},tics=5]{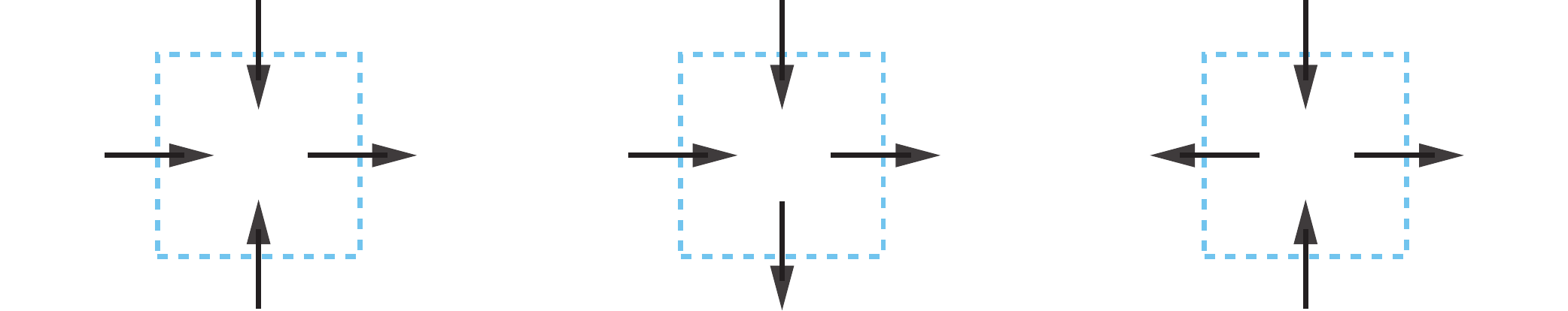}
\put(14,-5){{\small (a)}}
\put(48,-5){{\small (b)}}
\put(81,-5){{\small (c)}}
\end{overpic}
\end{center}
\caption{\label{fig:cell-flow} Three possible flow orientations that cover all possible
cases considering flow quantity (which may be zero) and symmetries (flipping of all flow 
directions and rotating the cell).}  
\end{figure}

To route the robots, we will only use the center ``+'' area of $d_g$ width of each 
$5d_g\times 5d_g$ cell. Fig.~\ref{fig:route-31} illustrates the routing plan for 
realizing the flow given in Fig.~\ref{fig:cell-flow}(a), which may be readily  
verified to be correct using basic algebra (i.e., assuming the top, left, and bottom
routes contain $x, y$, and $z$ flows, respectively, such that $x+y+z \le d_g$);
we omit the inclusion of the straightforward argument here. 
For arranging the robots, for horizontal cell boundaries, robots are aligned left. 
For vertical boundaries, robots are aligned toward the top. We note that if 
Fig.~\ref{fig:cell-flow}(a) is rotated, some adjustments are needed due to this 
choice of robot alignment but the change is minimal. Such alignments are 
necessary to ensure that the robot movements at cell boundaries match. 
\begin{figure}[h]
\begin{center}
\begin{overpic}[width={\ifoc 2.5in \else 2.5in \fi},tics=5]{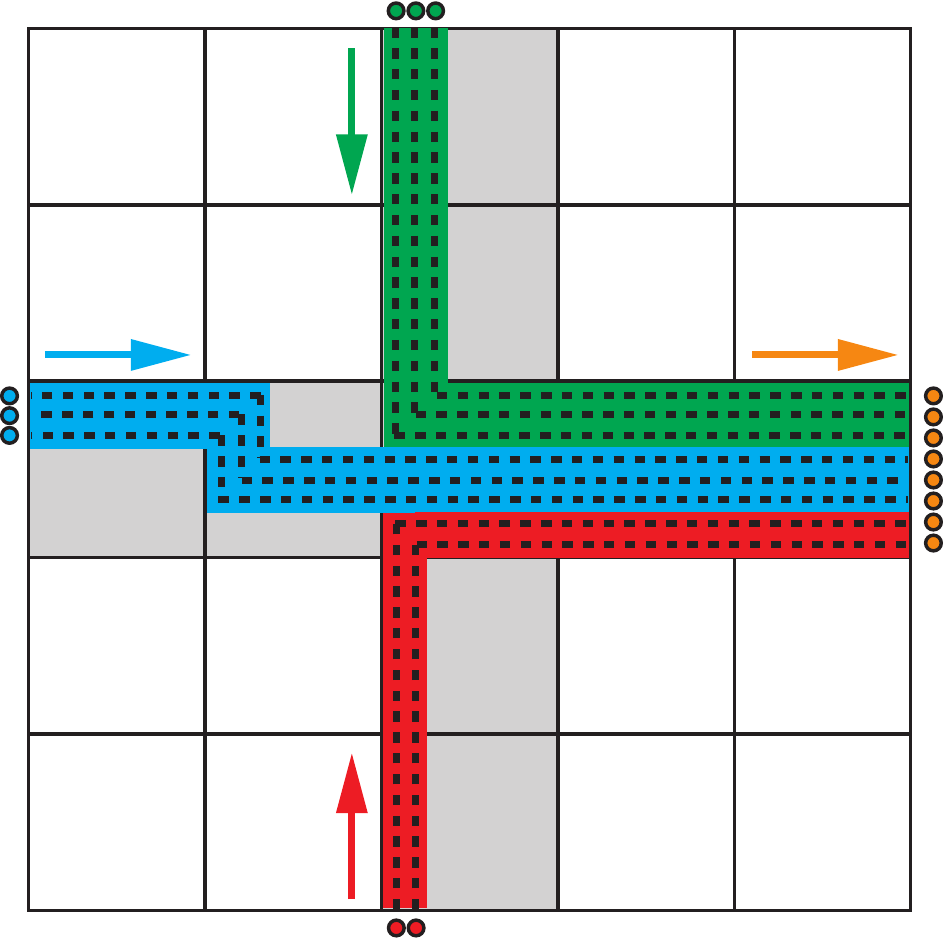}
\end{overpic}
\end{center}
\caption{\label{fig:route-31} Illustration of how a flow of size $8$ may be translated
to plans for robots for the case shown in Fig.~\ref{fig:cell-flow}(a). Each small square 
is of size $d_g \times d_g$.}  
\end{figure}

It is important to emphasize that we construct the paths so that for the 
incoming and outgoing $d_g \times d_g$ boundary areas of the ``+'' that are involved, 
robots only move straight through it, which is not necessary but simplifies 
things when we put multiple 
batches together later. For the cases from Fig.~\ref{fig:cell-flow}(b) and (c), 
illustrations of feasible routing plan construction are given in 
Fig.~\ref{fig:route-2}. Again, the incoming and outgoing flows follow straight lines 
in the $d_g\times d_g$ boundary areas of the ``+'' region. Because at most a $d_g$ 
amount of flow is being handled at a time per cell, the incoming flows can always 
be aggregated into the center $d_g \times d_g$ area before they get distributed 
outward to exit the cell. 
\begin{figure}[h]
\begin{center}
\begin{overpic}[width={\ifoc 4in \else 3.5in \fi},tics=5]{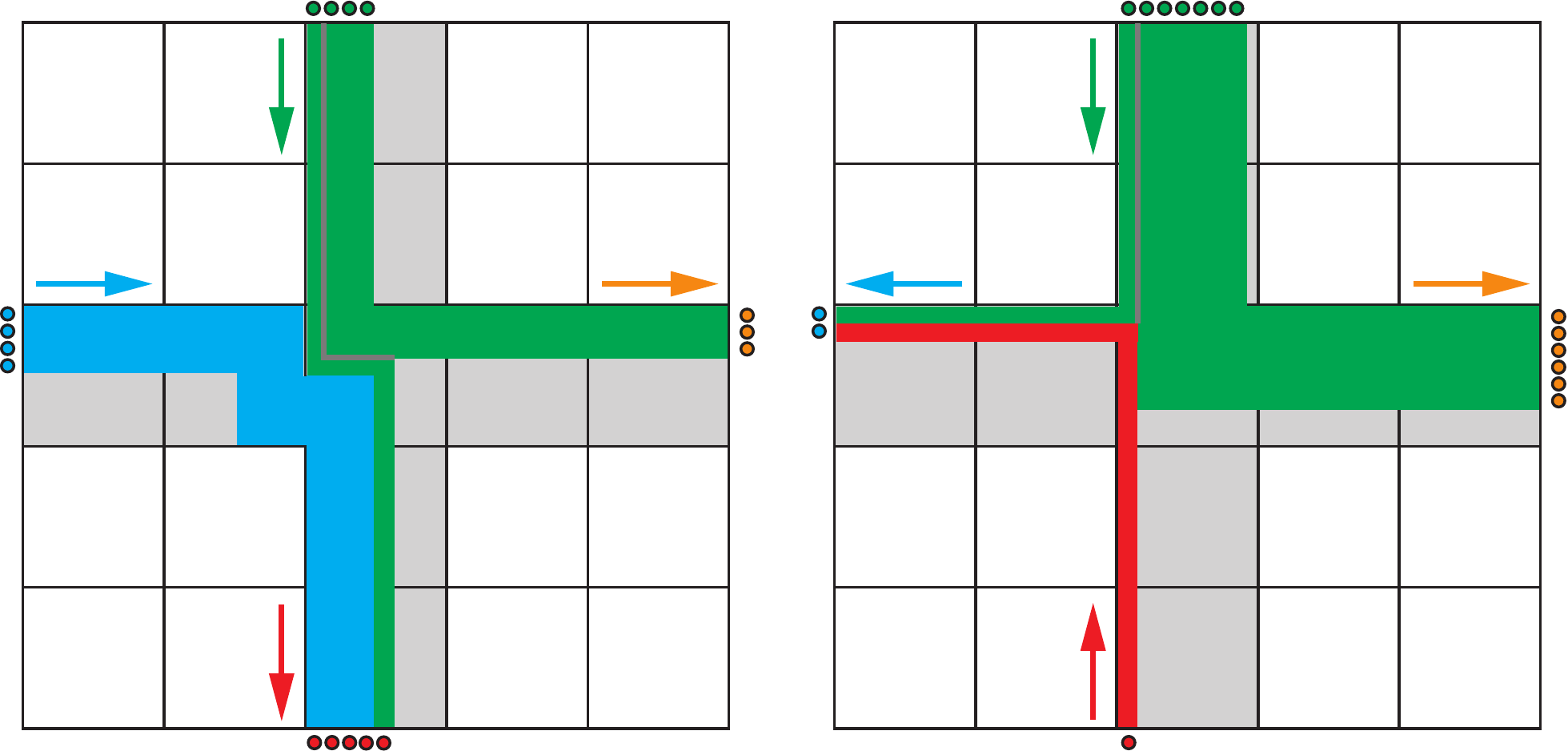}
\end{overpic}
\end{center}
\caption{\label{fig:route-2} Illustration of how flows be translated into feasible 
robot movements for the cases in Fig.~\ref{fig:cell-flow}, (b) and (c), respectively.
Each small square is of size $d_g \times d_g$. Because the incoming flows add up to 
no more than $d_g$, it is always possible to aggregate them into the center $d_g \times 
d_g$ area before sending them out of the cell.  
}
\end{figure}

Because there are only a constant number of flow arrangements for a cell (e.g., the 
three cases from Fig.~\ref{fig:cell-flow} plus some symmetric variants), there are 
only a constant number of possible parametrized routing plans. To compute such a plan, 
we note that each path is specified by a constant number of parameters. Together, this 
implies that the construction of the required paths for routing the robots in each 
cell only require a single pass through the cell, doable in $O(d_g^2)$ 
time. For $q_1q_2$ cells, the total is $O(q_1q_2d_g^2) = O(m_1m_2) = O(|V|)$.
\end{proof}

With a subroutine to push through $G$ a batch of up to $d_g$ unit circulations each 
step, $d_g$ such batches may be further grouped for sequential execution, allowing
the handling of up to $d_g^2$ at a time. This is established in the following lemma. 
\begin{lemma}[Multi-Batch Global Flow Routing]\label{l:batch}Up to $d_g^2$ 
unit circulations on the $q_1 \times q_2$ skeleton grid can be routed through $G$ using 
$O(d_g)$ makespan and $O(m_1m_2d_g)$ time. 
\end{lemma}
\begin{proof}
For the proof, we only need to focus on a single $d_g\times d_g$ boundary area of a 
single cell; all other boundaries and cells will be handled similarly. Moreover, 
we only need to worry about robots moving out of a cell due to symmetry. With these
reductions, we outline how to push up to $d_g^2$ robots out of the right boundary of
of the ``+'' region of a cell, which is a $d_g\times d_g$ grid. Call this $d_g\times d_g$ 
grid $c$. After dicing up the $d_g^2$ circulations into $d_g$ of $d_g$ sized batches, 
we invoke Lemma~\ref{l:gfrsb} to generate feasible routing plans for each $d_g$ sized 
batch. Because Lemma~\ref{l:gfrsb} guarantees that the generated paths are straight 
lines from left to right inside $c$, these batches can be sequentially arranged one 
after another for execution. An example for $d_g =6$ is illustrated in Fig.~\ref{fig:flow-arrange}
with each color representing a $d_g$ sized batch to be moved out through the right in 
one step. It is straightforward to check that the batches, when arranged into configurations
as illustrated in Fig.~\ref{fig:flow-arrange}(c), can be readily executed sequentially. 
In particular, once the paths for earlier batches are completed, the execution itself also 
prepares the next batch for execution (see Fig.~\ref{fig:batch-execution}). 
\begin{figure}[h]
\begin{center}
\begin{overpic}[width={\ifoc 4in \else 3.5in \fi},tics=5]{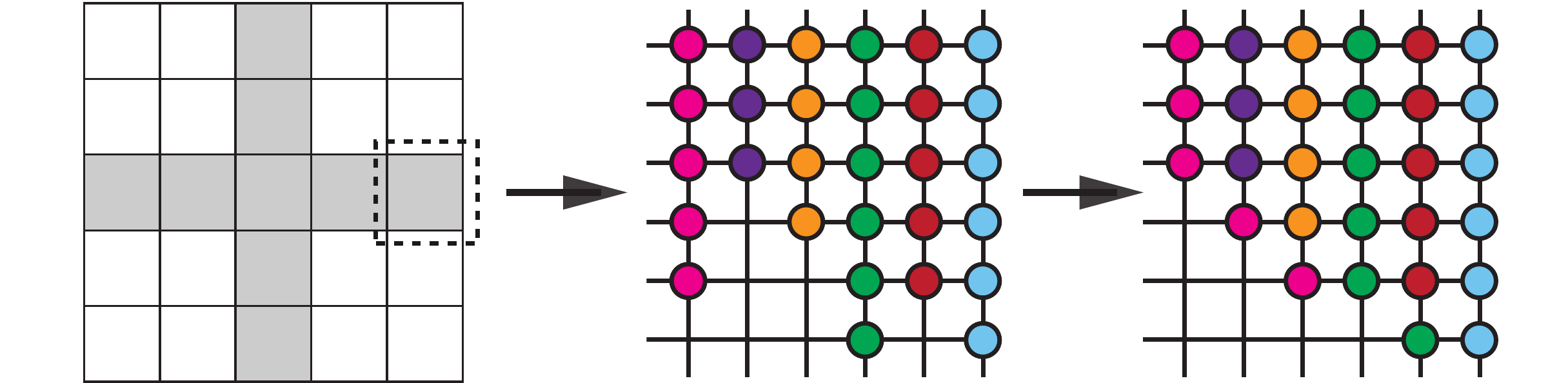}
\put(15.5,-4){{\small (a)}}
\put(50.5,-4){{\small (b)}}
\put(81.5,-4){{\small (c)}}
\end{overpic}
\end{center}
\caption{\label{fig:flow-arrange} (a) We are to route $d_g^2$ circulations 
through the right boundary of a cell in a $d_g \times d_g$ area, highlighted
with the dashed square. (b) The plans generated for the $d_g$ of $d_g$ sized
batches are arranged so that earlier plans appear on the right. For later 
plans, part of it get truncated. (c)  The further compacted batches for actual 
execution. For robots that are not shown, they will stay in the cell and have
no impact on the plan execution.}
\vspace*{-4mm}
\end{figure}
\begin{figure}[h]
\begin{center}
\begin{overpic}[width={\ifoc 4in \else 3.5in \fi},tics=5]{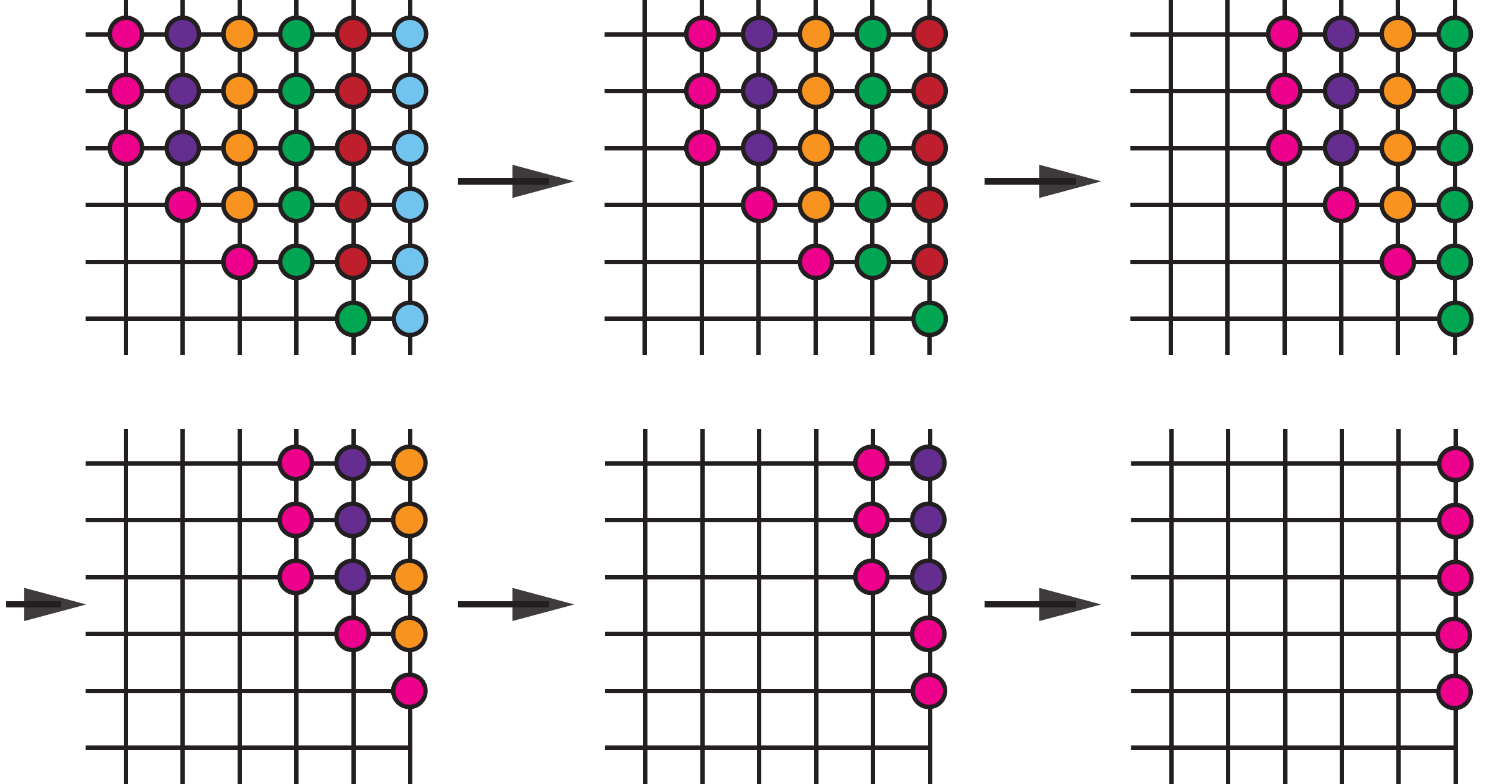}
\end{overpic}
\end{center}
\caption{\label{fig:batch-execution} Illustration of sequential execution of
$d_g = 6$ $d_g$ sized batches.}
\end{figure}

For computation time, for the $d_g^2$ circulation, we need to invoke the 
procedure from Lemma~\ref{l:gfrsb} for all $q_1q_2$ cells $d_g$ times, which 
incur a cost of $O(q_1q_2d_g^3) = O(m_1m_2d_g)$ running time, mostly used to 
write down the paths. To be able to actually prepare a cell for execution, 
\isag must be invoked on the cell once, which takes 
$O(m_1m_2d_g)$ time over all cells. This is 
the dominating term. 
\end{proof}

We will now complete proving Theorem~\ref{t:constant-time-optimal}. 

\begin{proof}[Proof of Theorem~\ref{t:constant-time-optimal}]
For the case of $d_g = o(m_2)$, on a $q_1\times q_2$ skeleton grid $G_S$ with each 
node being a $5d_g \times 5d_g$ cell, we first apply Lemma~\ref{l:orientation} to 
ensure that flows of robots across cell boundaries are uni-direction without 
diagonal movements, in $O(m_1m_2d_g)$ time. Then, Corollary~\ref{c:fd} computes a 
decomposition of the flow into up to $12d_g^2$ (vertex) unit circulations, in 
$O(m_1m_2d_g^2)$ time. Invoking Lemma~\ref{l:batch} a constant number of times, in 
$O(m_1m_2d_g)$ running time, we may globally route the robots so that all robots 
will be in the cell where its goal belongs to. We are then left with solving an 
\mpp for each individual cell, which again requires $O(m_1m_2d_g)$ running time 
over all cells. Putting this together with the cases handled by Lemma~\ref{l:2dspecial}, 
we concluded that an \mpp instance can be solved with $O(d_g)$ makespan in 
$O(m_1m_2d_g^2)$ time. If we use the randomized algorithm \cite{goel2013perfect} for 
matching, then the running time becomes  $O(m_1m_2d_g + m_1m_2\log\frac{m_1m_2}{d_g^2})$ 
expected time.
\end{proof}

%We remark here that \paf and results from \cite{demaine2018coordinated} for the 
%same purpose share some common traits, which is not surprising: for constructing a 
%$O(d_g)$ makespan algorithm, combining divide-and-conquer with a baseline algorithm 
%is indeed a rather natural first line of attack. 
Due to the dimension ignorant flow decomposition algorithm (Theorem~\ref{t:fd}), 
the running time for \paf in two dimensions incurs some additional cost 
over the comparable algorithm from \cite{demaine2018coordinated}.
On the other hand, the more general decomposition, coupled with \isag which 
directly supports arbitrary dimensions, enables the extension of \paf to 
three and higher dimensions. 

\section{\pafalgo in Higher Dimensions}\label{section:high-dimensions}
The overall \paf strategy for two dimensions generalizes to three and higher 
dimensions except when it comes to turn the decomposed flows into actually 
routing plan. We first establish that routing for the decomposed robot flow 
can be achieved for three dimensions with full details and then briefly 
discuss the necessary steps for extending it to arbitrary dimensions. 

\subsection{Three Dimensions}\label{section:3D}
On an $m_1 \times m_2 \times m_3$ grid, we first examine the main case of $d_g 
= o(m_3)$. For 3D, we will use a partition of cells of sizes $9d_g \times 9d_g 
\times 9d_g$ and assume that $9d_g$ divides $m_i$, i.e., $m_i = q_i9d_g, 1 
\le i \le 3$. It is straightforward to verify that \paf in 2D carries over 
except it is not clear how to route $d_g^3$ flow through the faces of a 
$9d_g \times 9d_g \times 9d_g$ cell, which requires the routing of 
$\Theta(d_g^2)$ flow in a single step. Generating paths for routing robots 
corresponding to the flow is significantly more involved than in the 2D case. 
In 2D, on a $d_g\times d_g$ grid, it is always possible to find up to $n \le d$ 
vertex disjoint paths that route $n$ robots through the grid (see 
Fig.~\ref{fig:2d3d-route}(a) for an illustration). These $n$ vertex disjoint 
paths then yield routes for routing the $n$ flow of robots. This may be 
readily proven via the observation that any {\em vertex cut} that isolates the 
incoming and outgoing batches of robots of size $n$ must have size at least 
$n$. On the other hand, for three and higher dimensions, it is no longer 
the case that robots on a face of a grid can be routed through the grid while
morphing the its shape, as illustrated with a counterexample in 
Fig.~\ref{fig:2d3d-route}(b). It can be shown that through a $k\ge 3$ 
dimensions grid with side lengths $d_g$, it is not always possible to find 
vertex disjoint paths for routing $d_g^{k-1} - d_g^{k-3}$ robots (for 3D, 
this number is $d_g^2 - 1$; for $d_g = 3$, this becomes $8$ as shown in 
Fig.~\ref{fig:2d3d-route}(b)). 
\begin{figure}[h]
\begin{center}
\begin{overpic}[width={\ifoc 4in \else 3.6in \fi},tics=5]{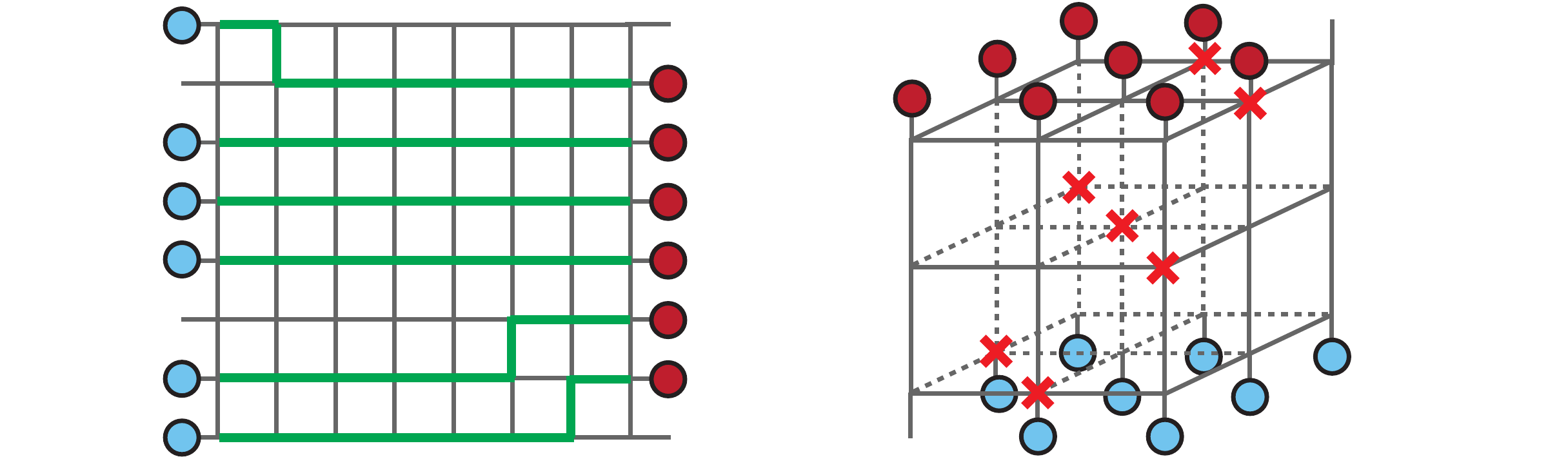}
\put(25.5,-4){{\small (a)}}
\put(72.5,-4){{\small (b)}}
\end{overpic}
\end{center}
\caption{\label{fig:2d3d-route} (a) Through a $d_g\times d_g$ grid, it is always 
feasible to find vertex disjoint paths that route up to $d_g$ robots through 
with arbitrary incoming and outgoing configurations. (b) The same is not true 
for three dimensional grids. Because the seven vertices marked by crosses 
isolate all possible paths between the two sets of eight robots, it is not 
possible to find eight paths that route the robots through.}
\end{figure}

For three (and higher) dimensions, we first systematically match the incoming 
flows into and the outgoing flows from a cell. In three dimensions, we match 
the up to six incoming and outgoing flows through a cell so that at most one 
face sends flow to its opposite face. If there is a single pair of opposite 
faces with one having incoming flow and one having outgoing flow, nothing 
needs to be done. Otherwise, if there are multiple such face pairs, pick two 
arbitrary such pairs $a_1, a_2, b_1$, and $b_2$. Without loss of generality, 
assume $f_{a_1} > 0, f_{a_2} < 0, f_{b_1} > 0$, and $f_{b_2} < 0$ (this is 
similar to the case illustrated in Fig.~\ref{fig:cell-flow}(b)). If $f_{a_1} 
\le |f_{b_2}|$, then we route all $f_{a_1}$ flow into $a_1$ to go out through 
$b_2$, which then avoids the need for routing any flow into $a_1$ to go out 
from $a_2$. If $f_{a_1} > |f_{b_2}|$, we do the same, which means that no flow 
from $b_1$ needs to go out through $b_2$. Either way, we effectively get rid 
of a dimension $i$ where $f_{i_1}*f_{i_2} < 0$. Doing this iteratively then 
leaves at most one such dimension where we may need to route any flow between 
the two opposite faces associated with that dimension. 

We now show how we may route flow from one face to other five faces through a 
$9d_g \times 9d_g \times 9d_g$ cell. Without loss of generality, we will show 
how to route flow coming in from the top face to the right face. Routing to 
the opposite face will be briefly explained afterward. 
We will route the flow to go through the center $d_g\times d_g$ regions on the 
six faces of the $9d_g \times 9d_g \times 9d_g$ cell and assume that a protocol 
is agreed on how the flow will be {\em shaped} between difference cells so the 
robot flows can be matched at cell boundaries. For example, on the top face, 
the $d_g^2$ flow may be ordered row by row (e.g., the $24$ robots on the top of 
Fig.~\ref{fig:breakdown}(a)), which result in a contiguous 2D shape inside 
a $d_g\times d_g$ region. Depending on the flow routing plan, this up to 
$d_g^2$ amount of flow is partitioned into 5 pieces (left, right, front, 
back, and center). We note that these pieces can again be made contiguous and 
in particular do not {\em interlock} with each other (bottom of 
Fig.~\ref{fig:breakdown}(a)). Based on the partition, the proper amount of 
flow to each face is then {\em pivoted} to go sideways row by row 
(see left figure of Fig.~\ref{fig:breakdown}(b)), except for flow that goes 
to the opposite face.
%, which is always feasible. 
\begin{figure}[h]
\vspace*{-1mm}
\begin{center}
\begin{overpic}[width={\ifoc 4in \else 2.8in \fi},tics=5]{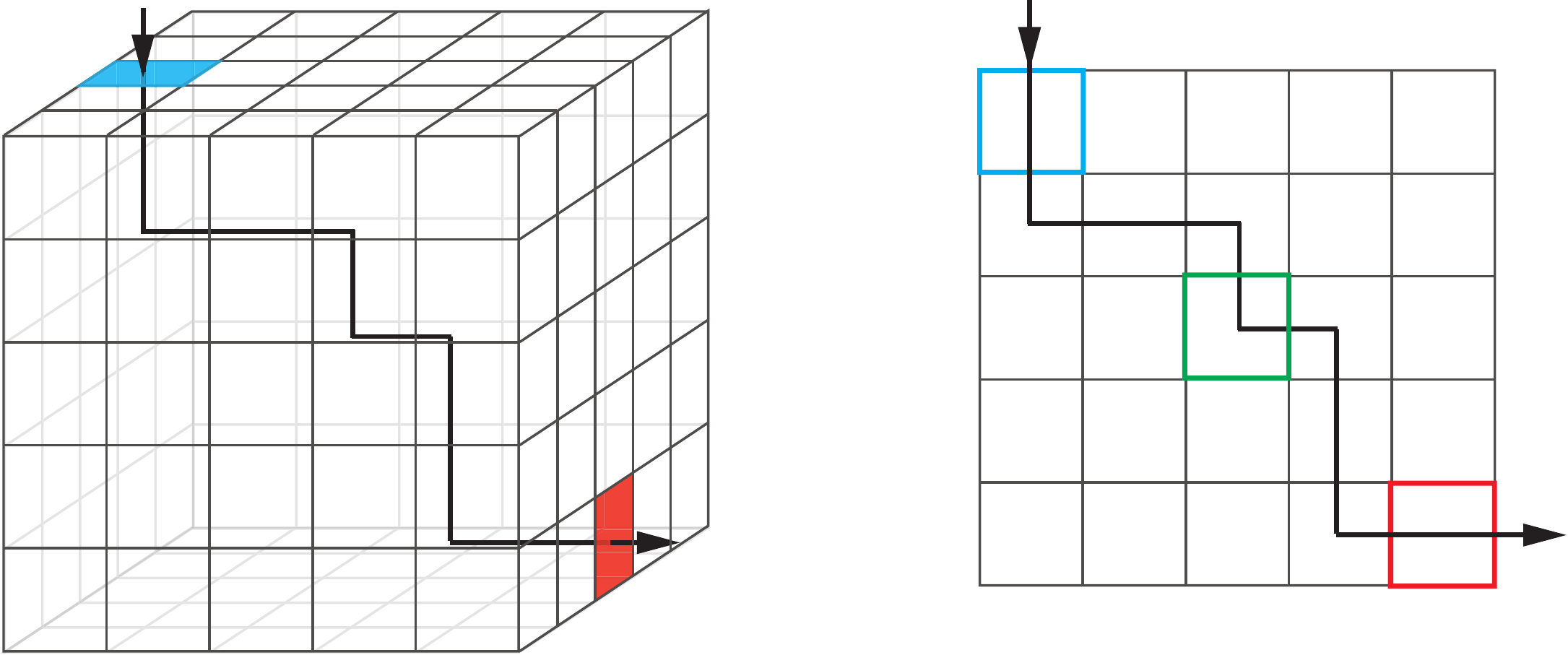}
\put(19,-5){{\small (a)}}
\put(76,-5){{\small (b)}}
\end{overpic}
\end{center}
\vspace*{-2mm}
\caption{\label{fig:sideways} Illustration of how a certain amount of flow may 
be routed sideways. Only the top-right-middle $5d_g\times 5d_g\times 5d_g$ portion 
of the cell is shown in (a). (b) is a projective view from the front. }  
\vspace*{-2mm}\end{figure}

For the flow going to the right face, we rearrange them to a row-majored shape 
using a $2d_g \times d_g \times d_g$ grid, as illustrated in Fig.~\ref{fig:breakdown}(b). 
At this point, we note that by symmetry, the same procedure can be applied to the flow
going out of the right face in the reverse direction. Using a $d_g \times d_g \times d_g$
grid  (the green one in Fig.~\ref{fig:sideways}(b) and Fig.~\ref{fig:breakdown}(b)) as 
a buffer zone, these two separately crafted routes can be perfectly matched, 
completing the routing plan for a pair of faces. For routing flow to an opposite face, we 
simply let the flow to go down two more $d_g \times d_g \times d_g$ grids after going
through the blue $d_g \times d_g \times d_g$ grid, after which we can do the same 
reshaping procedure. Once we can route $d_g^2$ flow through using a single step, 
we can do $d_g$ batches of these, pushing $d_g^3$ flow in $O(d_g)$ makespan. 
\begin{figure}[h]
\begin{center}
\begin{overpic}[width={\ifoc 4in \else 3.45in \fi},tics=5]{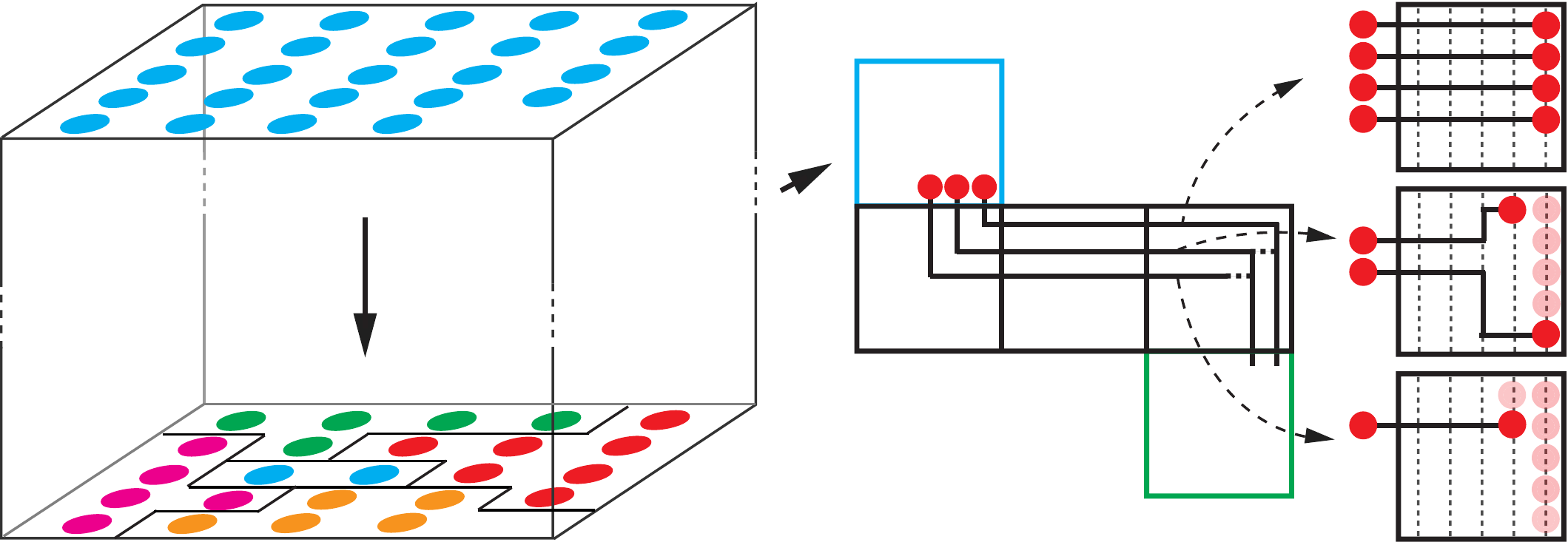}
\put(21,-5){{\small (a)}}
\put(76,-5){{\small (b)}}
\end{overpic}
\end{center}
\vspace*{-1mm}
\caption{\label{fig:breakdown} (a) Incoming $d_g^2$ flow may be broken into non-interlocking
pieces going to difference faces. This $d_g\times d_g\times d_g$ grid corresponds to the cyan
topped grid in Fig.~\ref{fig:sideways}(a).(b) A projective view (from the front) of how the 
three rows of red robots can be routed and reshaped into two row-major ordered rows, going
downwards.}  
\vspace*{-2mm}
\end{figure}

Our main goal so far is to show that it is feasible to route $d_g^3$ flow in $O(d_g)$ make
span. To actually create the routing plan, we apply the max-flow algorithm (e.g., 
\cite{ford1956maximal}) to an augmented direct graph generated on the $9d_g \times 9d_g 
\times 9d_g$ grid via vertex splitting, a standard technique used in finding vertex 
disjoint paths. We summarize the results in the following theorem. 

\begin{theorem}[\paf in Three Dimensions]\label{t:3dpaf}Let $G = (V, E)$ be an $m_1 
\times m_2 \times m_3$ grid. Let $p$ be an arbitrary \mpp instance on $G$. A solution 
with $O(d_g(p))$ makespan can be computed in $O(d_g^3|V|)$ and $O(|V|^2)$ time.
\end{theorem}
\begin{proof}
Without loss of generality, we may assume that $m_1$, $m_2$, and $m_3$ are desired 
multiples of $d_g$ as needed. 

For $k = 3$, $d_g$, when compared with $m_1, m_2$, and $m_3$ (we remind the 
reader that $m_1 \ge m_2 \ge m_3$ and we assume $d_g = o(m_1)$), raises 
three cases: {\em (i)} $d_g = o(m_1)$ and $d_g = \Omega(m_2)$, 
{\em (ii)} $d_g = o(m_2)$ and $d_g = \Omega(m_3)$, and {\em (iii)} $d_g = o(m_3)$. 
We note that $O(d_g^3|V|)$ and $o(|V|^2)$ do not necessarily imply each other in all 
cases.

For the case of $d_g = o(m_1)$ and $d_g = \Omega(m_2)$, a result similar
to Lemma~\ref{l:2dspecial} can be proved, over a partition of $G$ as illustrated 
in Fig.~\ref{fig:dgomegam2m3}. The overall makespan is readily verified as $O(d_g)$. 
The only difference, as compared with the proof for Lemma~\ref{l:2dspecial}, is that 
the 3D version of $\isag$ needs to be used. By Theorem~\ref{t:isag}, solving a 
single 3D \isag on a $d_g\times m_2 \times m_3$ grid takes time $O(d_g^2m_2m_3 + 
d_gm_2^2m_3 + d_gm_2m_3^2)$. Making a total $\frac{m_1}{d_g}$ of such parallel calls 
takes time $O(d_gm_1m_2m_3) = O(d_g|V|)$, which is both $O(d_g^3|V|)$ and $o(|V|^2)$. 
\begin{figure}[h] 
\begin{center}
\begin{overpic}[width={\ifoc 4in \else 3.5in \fi},tics=5]{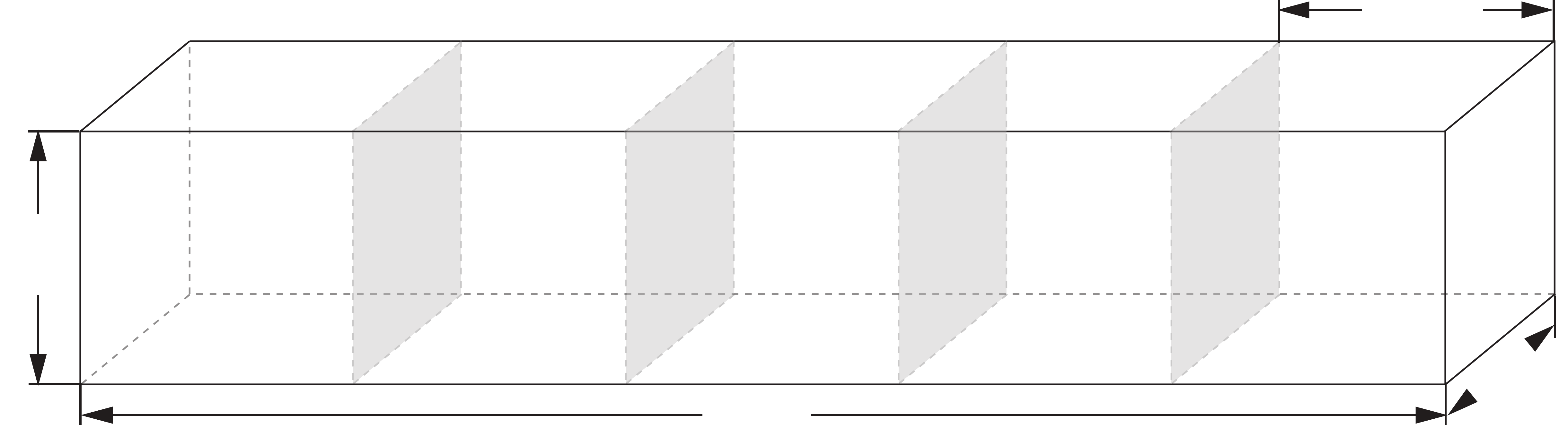}
\put(0,10){{\small $m_2$}}
\put(46,0){{\small $m_1$}}
\put(93.2,1.7){\rotatebox{40}{{\small $m_3$}}}
\put(88.5,26.5){{\small $d_g$}}
\end{overpic}
\end{center}
\vspace*{-3mm}
\caption{\label{fig:dgomegam2m3} In the case of $d_g = o(m_1)$ and $d_g = 
\Omega(m_2)$, $G$ may be partitioned blocks of size $d_g\times m_2 \times m_3$.}
\end{figure}
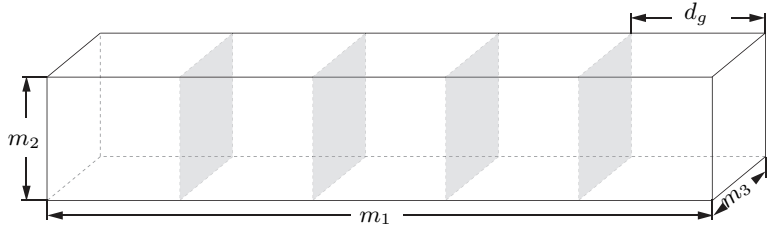

For the case of $d_g = o(m_2)$ and $d_g = \Omega(m_3)$, we partition $G$
into cells of size $5d_g\times 5d_g\times m_3$ each, as illustrated in 
Fig.~\ref{fig:dgomegam3}, to get a $q_1 \times q_2$ 2D skeleton grid $G_S$ 
of these cells. For solving the partitioned problem, we essentially follow 
the main case of \paf in 2D. The procedures for carrying out diagonal rerouting 
and flow cancellation (Lemma~\ref{l:orientation}) can be executed as is 
on $G_S$ using 3D \isag. For flow decomposition on $G_S$, instead of up to 
$6d_g^2$ flow, the 3D case now has up to $6m_3d_g^2$ flow through a cell 
boundary. The same flow decomposition procedure (i.e., Corollary~\ref{c:fd}) 
can nevertheless be carried out to decompose $O(m_3d_g^2)$ flow on $G_S$ 
into unit circulations, as $O(d_g)$ of $m_3d_g$ sized batches. 
This is because the boundary between two cells is now a $d_g \times m_3$ 2D
grid and can allow $m_3d_g$ robots to pass through at a single step. The 
conversion of the flow into executable paths for global robot routing can 
then be completed using the 3D extended version of Lemma~\ref{l:gfrsb} and 
Lemma~\ref{l:batch}; adding a dimension orthogonal to the $d_g\times d_g$ 
grid is straightforward. After the global routing step, each robot resides in 
a cell where its goal also resides; a parallel call to 3D \isag on all 
individual cells then solves the problem. 
\begin{figure}[h]
\begin{center}
\begin{overpic}[width={\ifoc 4in \else 3.5in \fi},tics=5]{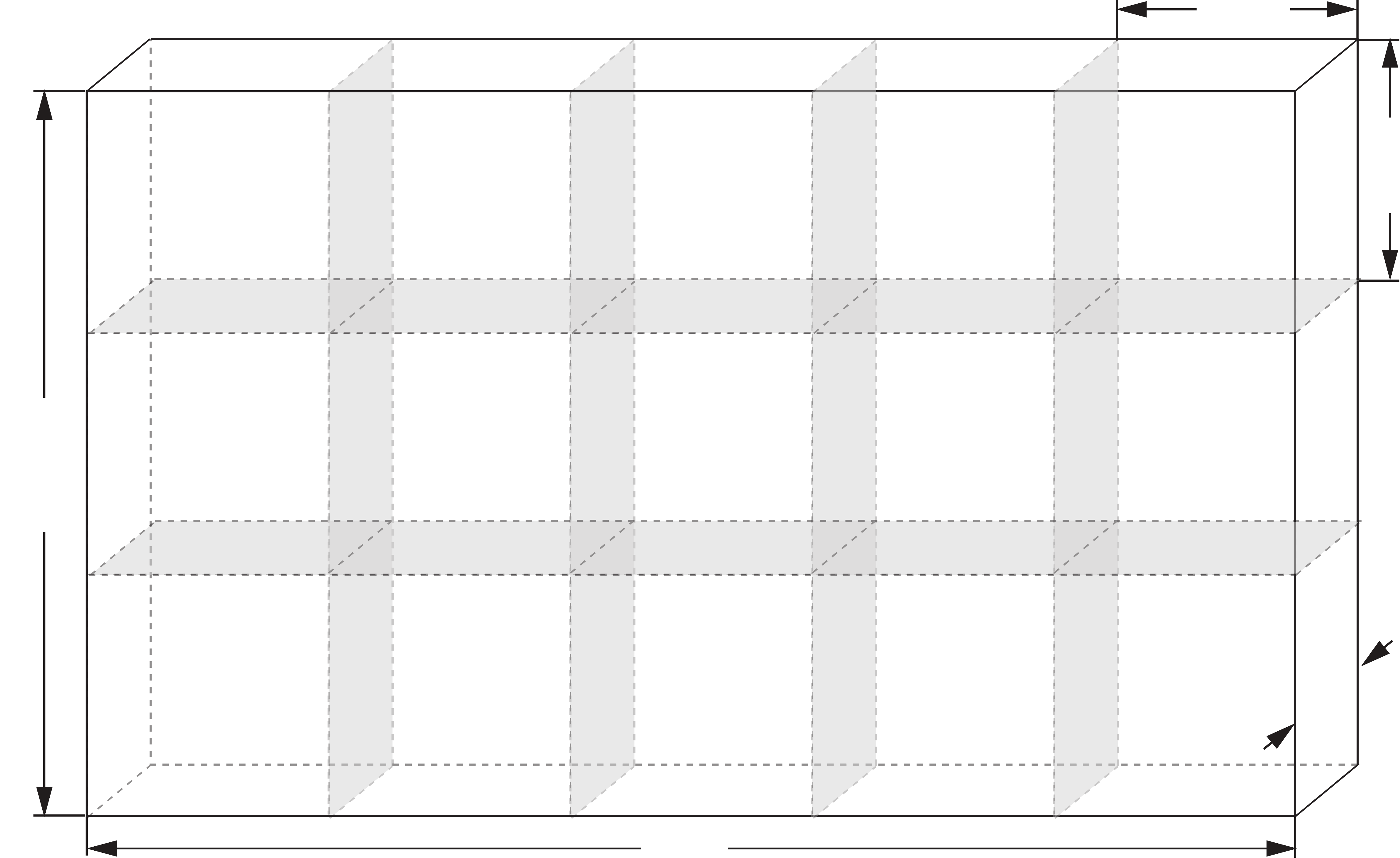}
\put(0,28){{\small $m_2$}}
\put(46.5,0){{\small $m_1$}}
\put(92.5,9.5){\rotatebox{40}{{\small $m_3$}}}
\put(86,60){{\small $5d_g$}}
\put(97.5,47){\rotatebox{90}{{\small $5 d_g$}}}
\end{overpic}
\end{center}
\caption{\label{fig:dgomegam3} In the case of $d_g = o(m_2)$ and $d_g = 
\Omega(m_3)$, $G$ may be partitioned into blocks of size $5d_g\times 5d_g \times m_3$.}
\end{figure}
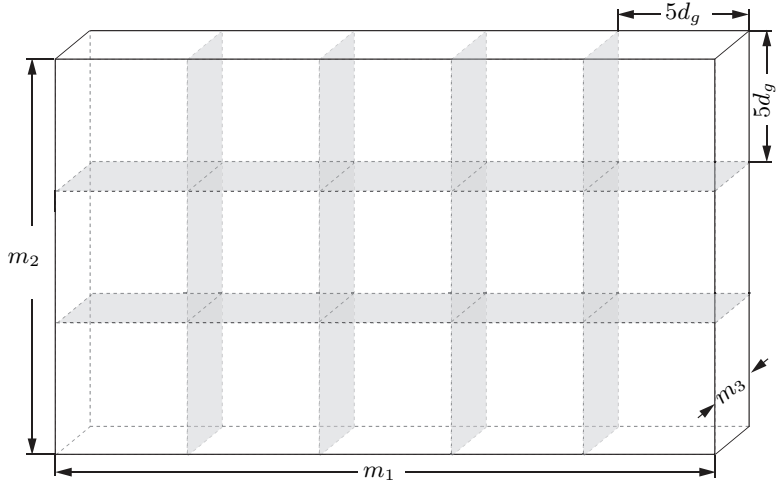

Again, the resulting makespan for this case is clearly $O(d_g)$.
Running time wise, the flow decomposition needs to route $f = O(m_3d_g^2)$ flow on 
a skeleton grid with $\frac{m_1m_2}{d_g^2}$ edges, which requires a total 
running time of $O(m_1m_2m_3^2d_g^2)$. 
The other cost is to invoke 3D \isag in parallel on $\frac{m_1m_2}{d_g^2}$ of 
$5d_g\times 5d_g \times m_3$ sized cells, which by Theorem~\ref{t:isag} takes 
time $O(\frac{m_1m_2}{d_g^2}(d_g^3m_3 + d_g^3m_3 + d_g^2m_3^2)) = O(d_gm_1m_2m_3)$. 
The overall running time is then $O(d_g^2m_1m_2m_3^2)$. Because $d_g = o(m_2)$ and 
$d_g = \Omega(m_3)$, the running time is both $O(d_g^3|V|)$ and $o(|V|^2)$. 

For the main case of $d_g = o(m_3)$, we partition $G$ into $9d_g \times 9d_g \times9
d_g \times$ cells. Assuming $m_1 = 9q_1g_d, m_2 = 9q_2g_d$, and $m_3 = 9q_3g_d$, this 
yields a 3D skeleton grid $G_S$ of dimensions $q_1 \times q_2 \times q_3$ (we omit 
the pictorial illustration of the case, which is difficult to visually observe). In this 
case, each cell $c$ may interface with $26$ other cells, with $6$ of these neighbors 
each sharing a $9d_g\times 9d_g$ boundary with $c$. The rest $20$ neighbors of $c$ are 
diagonal neighbors of some form, either along a $9d_g$ length edge ($12$ of these, 
denoted as edge-diagonal neighbors) or a single vertex ($8$ of these, denoted as 
vertex-diagonal neighbors). 

For the main case, we verify that the flow orientations steps 
(Lemma~\ref{l:orientation}) carry over with minor modifications: more calls to 
\isag is required for each cell due to the increased number of neighbors and 
3D diagonal rerouting needs to first convert vertex-diagonals to edge-diagonals.
The number of such parallel calls to \isag remains constant, however, retaining 
the $O(d_g)$ makespan guarantee and actually reduces the asymptotic running time.
The flow decomposition step (Corollary~\ref{c:fd}) extends with the flow amount 
being $O(d_g^3)$ per cell, decomposed into $O(d_g)$ of $d_g^2$ sized batches. 

To construct a global routing plan for realizing these batches, instead of using 
manual construction as we have done with Lemma~\ref{l:gfrsb} and Lemma~\ref{l:batch}, 
in the 2D case, we directly apply max-flow to generate the vertex disjoint paths 
for routing the robots. After the flow decomposition step, each batch of up to
$d_g^2$ flow, by our earlier argument, is always possible to be routed through a 
$9d_g \times 9d_g\times 9d_g$ cell. We may invoke the Ford-Fulkerson algorithm 
\cite{ford1956maximal} on an auxiliary graph (through vertex splitting, a standard
technique) of the $9d_g \times 9d_g\times 9d_g$ cell to obtain up to $d_g^2$ vertex 
disjoint paths, which route that many robots in a single step through the cell. 
It is possible to bundle $d_g$ of $d_g^2$ sized batches together, which incur a 
makespan of $d_g$. Therefore, similar to the 2D case, the global robot routing 
can be completed using $O(d_g)$ makespan. 

With the global routing of robots completed, we again end up with the case
that every robot is now in a cell where its goal also belongs to. \isag can 
then be invoked to solve the problem in parallel. 
Following similar analysis as in the 2D case, the algorithm produces an $O(d_g)$ 
makespan solution. For running time, there are three main costs: {\em (i)} \isag calls, 
{\em (ii)} matching for flow decomposition, and {\em (iii)} max-flow based global 
robot routing plan generation. For {\em (i)}, $q_1q_2q_3$ parallel calls to 3D \isag 
is needed, demanding a time of $O(q_1q_2q_3d_g^4) = O(d_g|V|)$. For {\em (ii)}, flow 
decomposition is now performed on $O(d_g^3)$ flow on a graph with $O(q_1q_2q_3d_g^3)$
edges, requiring $O(d_g^3|V|)$ time. For {\em (iii)}, using Ford-Fulkerson 
\cite{ford1956maximal}, the total running time is $O(q_1q_2q_3d_g^6) = O(d_g^3|V|)$, 
which is $o(|V|^2)$. 
\end{proof}

%We now visit the claim that our algorithm runs in sub-quadratic time except in 
%the most degerate case of $d_g = \Omega(|V|)$.  For \paf, we note that its 
%running time $O(kd_g^{\frac{k}{2}}|V|^{\frac{3}{2}}) = o(|V|)^2$ for 
%$d_g = o(m_k)$ because the factor $k$ is readily absorbed as we replace 
%$d_g^{\frac{k}{2}}$ with $|V|^{\frac{1}{2}} = 
%\Theta((9d_g)^{\frac{k}{2}}(\prod q_i)^{\frac{1}{2}})$. For other sub cases 
%when $d_g = o(m_1)$ holds, it can also be readily verified that the 
%sub-quadratic bound holds. For example, in the proof of Theorem~\ref{t:3dpaf}
%the running time \paf when $d_g = o(m_1)$ and $d_g = \Omega(m_2)$ is 
%$O(d_gm_1m_2m_3 + d_g^{\frac{1}{2}}m_1m_2^2m_3 + 
%d_g^{\frac{1}{2}}m_1m_2^{\frac{3}{2}}m_3^2)$, which is $o(|V|^2)$. We omit 
%the full derivation here but point out that the conclusion cannot be directly 
%obtained through the general $O(kd_g^{\frac{k}{2}}|V|^{\frac{3}{2}})$ time 
%bound. That is, while the claim of $O(kd_g^{\frac{k}{2}}|V|^{\frac{3}{2}})$ 
%time and sub-quadratic time both hold, they do not imply each other. This 
%leaves us with the case of $d_g = \Omega(m_1)$. This case, \paf will directly 
%invoke \isag and yield the time bound given by~\eqref{f:sagmpp}, which 
%requires sub-quadratic time for fixed $k$ except when $m_1 = \Omega(|V|)$ 
%and $m_2 = \ldots = m_k = O(1)$, corresponding to the case of $G$ being nearly 
%one dimensional.

\subsection{Arbitrary Fixed Dimension}
Arguments from Section~\ref{section:3D}, in particular 
Theorem~\ref{t:3dpaf}, suggest that the overall \paf strategy applies to an 
arbitrary $k$-dimensional grid with $k \ge 2$ when two conditions are met. 
First, after partitioning the larger grid into $\Theta(d_g)$ sized cells, it 
must be possible to make local robot exchanges in $O(d_g)$ makespan so that 
in the leftover problem, robots in a given cell, say $c_i$, only have targets 
in cells that share a $(k-1)$-dimensional face with $c_i$. If this is possible, 
we may then apply the dimension-invariant Theorem~\ref{t:fd} to break down the 
(up to $\Theta(d_g^k)$) robot flow between adjacent cells into batches of size  
$O(d_g^{k-1})$. Second, after the decomposition, it must be possible to find
vertex disjoint paths that route up to $\Theta(d_g^{k-1})$ robots in a single 
step through a $\Theta(d_g)$ sized grid cell (e.g., Fig.~\ref{fig:route-31},
Fig.~\ref{fig:route-2}, and Fig.~\ref{fig:breakdown}). For $k$ dimensions, 
it is readily verified that the first condition holds through direct 
generalization of the relevant results from 
Section~\ref{section:two-dimensions}. This leaves us with finding vertex 
disjoint paths for routing the robots globally, i.e., generalization of 
Lemma~\ref{l:gfrsb} to $k$ dimensions. We outline how this may be realized as 
a further generalization of the 3D case (e.g., Fig.~\ref{fig:breakdown}), 
starting with the introduction of some necessary definitions. 

For convenience, we use the short hand $[x]^k$ to denote a cubical 
$k$-dimensional grid $x \times \ldots \times x$ (i.e., each dimension 
spans a path of length $x$). Such a grid has $2k$ faces where each face is a 
$(k-1)$-dimensional grid (which can be expressed as $[x]^{k-1}$). Let these 
faces be $f_1, \ldots, f_{2k}$. On a $2$-dimensional grid and an ordering of 
its two dimensions (or axes) $[d_1, d_2]$, we say a set $P$ of vertices on the 
grid is {\em $[d_1, d_2]$-regular} if $P$ is arranged such that, viewing $d_1$ 
as the number of columns and $d_2$ as the number of rows, $P$ fully occupies 
the first $\lfloor |P|/d_1\rfloor$ rows and then the first $|P| \mod d_1$ 
vertices of row $\lfloor |P|/d_1\rfloor + 1$. This is essentially a type of 
row-major ordering of $P$. On a $k$-dimensional grid (graph) $G$, for an 
ordering of its dimensions $[d_1, \ldots, d_k]$, we say set of vertices $P$ 
on $G$ is {\em $[d_1, \ldots, d_k]$-regular} if $P$ fully occupies the first 
$\lfloor |P|/(d_1\times \ldots d_{k-1})\rfloor$ layers of $G$ and then the 
rest of $P$ is $[d_1, \ldots, d_{k-1}]$-regular in the $(\lfloor 
|P|/(d_1\times \ldots d_{k-1})\rfloor + 1)$-th layer of $G$ (if applicable). 
We denote the geometric arrangement of $P$ on a $[x]^k$ grid as the {\em shape} 
of $P$. Our main goal is to {\em reshape} two point sets on a $[x]^k$ grid to 
match each other using vertex disjoint paths that go through the grid. 

\begin{lemma}[Vertex Disjoint Paths in High Dimensional Grids]\label{l:vdp-kd}On a 
$k$-dimensional grid $[cd_g]^k$ with $c$ being an odd integer constant, 
let $f_1$ and $f_2$ be two arbitrary faces of the grid and let $P_1$ and 
$P_2$ be two sets of points in the center $[d_g]^{k-1}$ area of $f_1$ and $f_2$, 
respectively, with $|P_1| = |P_2|$. Then for some proper $c$ independent of $k$, 
there are $|P_1|$ vertex disjoint paths within the grid that connect distinct 
elements from $P_1$ and $P_2$.
\end{lemma}
\begin{proof}[Proof sketch]The proof is via construction with the core idea
similar to the 3D case as illustrated in Fig.~\ref{fig:sideways} and 
Fig.~\ref{fig:breakdown}. At a higher level, for the incoming flow $P_1$ 
(which initially may assume an arbitrary shape in the center $[d_g]^{k-1}$ 
area of the face $f_i$), we construct vertex disjoint paths that {\em peel} 
the flow apart (i.e., separate the flow using vertex disjoint paths), one 
dimension a time, until we are left with 2D flows which we can easily {\em 
reshape}. We then {\em stack} these reshaped 2D flows recursively to 
eventually form a $(k-1)$-dimensional shape that is regular. From the 
outgoing flow side, the same procedure is performed, only in the reverse 
direction. Connecting the two halves together then produces a full routing 
plan as vertex disjoint paths. 

In the case of 3D, Fig.~\ref{fig:sideways} illustrates the reshaping a flow 
along the $y$-axis (pointing up) into a flow along the $x$-axis (pointing to 
the right). For the incoming flow part (Fig.~\ref{fig:breakdown}), the 2D 
shape (the red discs in Fig.~\ref{fig:breakdown}(a)) is first peeled into 1D 
shapes along the $xy$ dimensions; the $z$-coordinates do not change. The 1D 
flows are then individually reshaped using some $xz$-planes with fixed $y$ 
coordinates. Lastly, the reshaped 1D flows are stacked into a $[z,x]$-regular 
shape (as the flow enters green cube in Fig.~\ref{fig:breakdown}(b)). 
Similarly, the outgoing flow goes through the same process (in reverse) and 
is reshaped to be $[z,y]$-regular as the flow just exits the green cube. Within 
the green cube, the $[z,x]$-regular incoming flow is {\em pivoted} to match 
the $[z,y]$-regular outgoing flow. 

For dimension $k$ with the axes being $d_1, \ldots, d_k$, we may assume that 
$f_1$ and $f_2$ are orthogonal to $d_1$ and $d_2$ dimensions, respectively. 
We outline how to reshape the incoming flow $P_i$ to a regular shape; the 
outgoing flow portion is symmetric. The peelings are done recursively followed 
by recursive stackings. In the first peeling, we peel the $(k-1)$-dimensional 
shape $P_1$ (a set with size up to $d_g^{k-1}$) from the incoming flow along 
$d_1d_2$ dimensions, holding other coordinates fixed. This yields up to $d_g$ 
$(k-2)$-dimensional shapes, each of which occupies a unique $d_2\ldots d_k$ 
hyperplane with a unique $d_1$ value. Then, each $(k-2)$-dimensional shape, 
now living in its own $(d-1)$-dimensional hyperplane, is further peeled into 
$(d-2)$-dimensional shapes along $d_2d_3$ dimensions, holding other coordinates 
fixed. This then results $(d-3)$-dimensional shapes with unique $(d_1, d_2)$
coordinates (i.e., each of the shape again lives in a hyperplane disjoint from
one another). Repeating this procedure, we eventually go down to one-dimensional 
shapes, at which point we can reshape them arbitrarily and then stack them 
back recursively to get a $[d_2\ldots d_k]$-regular shape. 

Noting that the peeling operations produces shapes that occupy different 
hyperplanes within the grid and that each dimension is used in at most two 
peeling operations, we conclude that the peeling and subsequent stacking 
operations for reshaping $P_i$ can be completed within a $[3d_g]^k$ grid. 
The reshaping of the outgoing flow takes the same amount of space. Then, 
having a grid of size $[6d_g]^k$ between the center of the two faces $f_1$
and $f_2$ is sufficient to allow the reshaping of $P_1$ to $P_2$. Choosing 
$c$ to be $15$ then provides sufficient space for reshaping flows between two 
arbitrary faces (i.e., it is possible to fit a $[6d_g]^k$ grid between two 
arbitrary faces $f_1$ and $f_2$). 
\end{proof}

Lemma~\ref{l:vdp-kd} suggests that we can find vertex disjoint paths in a 
$[\Theta(d_g)]^k$ grid cell that route up to $d_g^{k-1}$ (robot) flow among 
the center $[d_g]^{k-1}$ regions of the $2k$ faces of the grid cell. This then 
allows \paf to work for an arbitrary dimension $k$. 

\begin{theorem}[\paf in $k$ Dimensions]\label{t:kdpaf}Let $G = (V, E)$ be an 
$m_1 \times \ldots \times m_k$ grid for some arbitrary but fixed $k \ge 2$. 
Let $p$ be an arbitrary \mpp instance on $G$. A solution with $O(d_g(p))$ 
makespan can be computed in $O(d_g^k|V|)$ and $O(|V|^2)$ time.
\end{theorem}
\begin{proof}[Proof sketch]The algorithm itself is a generalization of the 
3D case via induction. Here, we only analyze the time complexity of the main 
case, i.e., $d_g = o(m_k)$, which dominates other cases. In this case, the 
running time again boils down to three main contributors: {\em (i)} \isag 
calls, {\em (ii)} matching for flow decomposition, and {\em (iii)} max-flow 
based global robot routing plan generation. For {\em (i)}, a running time of 
$O(d_g|V|)$ is needed. For {\em (ii)}, $O(d_g^k|V|)$ time is needed. For 
{\em (iii)}, using Ford-Fulkerson \cite{ford1956maximal}, the total running 
time is also $O(d_g^k|V|)$. Therefore, the overall running time is 
$O(d_g^k|V|)$, which is sub-quadratic in $|V|$. 
\end{proof}

\section{Discussion}\label{section:discussion}
We conclude the paper discussing some natural extensions of \isag and \paf. 

%{\em Computation time}. For \paf, it runs in a sub-quadratic $o(|V|^2)$ whenever
%$G$ is non-generate or $d_g = o(m_1)$. In other words, \paf runs mostly in sub-quadratic
%time, sometimes much better, unless $G$ is degenerate and $d_g = \Omega(m_1)$ both
%hold. In this case, \paf runs in a quadratic $O(|V|^2)$ time. At the same time, 
%there are cases here that requires $\Omega(|V|^2)$ running time to simply write down 
%the solution. So in this sense, the running time of \paf is also tight. 

%{\em Dependency on the grid dimension}. As the dimension $k$ increases, our 
%algorithm yields makespan that is proportional to $9^k$, which may be undesirable. 
%This may be alleviated in multiple ways. For $k \ge 3$, we used a partition with
%a side length of $9d_g$, which is perhaps not necessary. For example, a more 
%refined argument seems possible to allow the reshaping of $d_g^{k-1}$ flow through
%a $[2d_g]^{k-1}\times d_g$ grid instead of the $[3d_g]^{k-1}\times d_g$ grid used 
%in Lemma~\ref{l:vdp}. Also, our current treatment of the problem only looks at 
%some coarse factors, namely $d_g$ and $m_i, 1 \le i \le k$. A more careful 
%examination of $X_I$ and $X_G$ may lead to output sensitive algorithms with 
%smaller required makespan. For example, if the number of robots that need to be
%moved is very small as compared to $|V|$, then the flow decomposition step may 
%only need to handle $o(d_g^k)$ amount of flow, which reduces required makespan.

{\em Extension to other grid-like graphs}. 
Our results have focused on the underlying graph $G$ being axis-aligned grids. As 
pointed out in \cite{yu2017constant}, the results developed in this paper readily 
apply to other types of grid-like graphs, e.g., honeycombs and grids with triangular 
faces. Indeed, as long as the graph admits some forms of Lemma~\ref{l:distribute} 
and Lemma~\ref{l:swap}, then a version of \isag can be derived for the setting. 
For the decomposed global flow to be routed effectively, some form of 
feasibility argument is needed, which can be ensured if the underlying graph can 
be partitioned into orthogonal dimensions. 

{\em Continuous domain}. As pointed out in \cite{yu2017constant} and with more 
details in \cite{HanRodYu2017Arxiv,demaine2018coordinated}, routing algorithms 
on grids also extend to continuous settings for the routing of 
identical sized disc robots (balls in higher dimensions) that may be packed 
arbitrarily close to each other. The extension is carried out with an expansion 
phase of the (continuous) initial configuration such that sufficient space is 
available for aligning the robots, as unlabeled ones, onto a grid for routing. 
This yields a grid $G$ and an associated $X_I$. The same is applied to the goal 
configuration, which will use the same $G$ and produces an $X_G$. Since the expansion 
only needs to grow volume occupied by the initial configuration by a constant (for 
fixed dimension $k$), the $O(1)$-approximation guarantee is then fully preserved. 

\vspace*{2mm}
\noindent\textbf{Acknowledgments}. The author would like to thank Pranjal 
Awasthi and Mario Szegedy for helpful discussions. 

%\textcolor{red}{\textbf{Add continuous extension?}}

\bibliographystyle{IEEEtran}
\bibliography{jingjin}

\begin{thebibliography}{10}
\providecommand{\url}[1]{#1}
\csname url@samestyle\endcsname
\providecommand{\newblock}{\relax}
\providecommand{\bibinfo}[2]{#2}
\providecommand{\BIBentrySTDinterwordspacing}{\spaceskip=0pt\relax}
\providecommand{\BIBentryALTinterwordstretchfactor}{4}
\providecommand{\BIBentryALTinterwordspacing}{\spaceskip=\fontdimen2\font plus
\BIBentryALTinterwordstretchfactor\fontdimen3\font minus
  \fontdimen4\font\relax}
\providecommand{\BIBforeignlanguage}[2]{{%
\expandafter\ifx\csname l@#1\endcsname\relax
\typeout{** WARNING: IEEEtran.bst: No hyphenation pattern has been}%
\typeout{** loaded for the language `#1'. Using the pattern for}%
\typeout{** the default language instead.}%
\else
\language=\csname l@#1\endcsname
\fi
#2}}
\providecommand{\BIBdecl}{\relax}
\BIBdecl

\bibitem{yu2017constant}
J.~Yu, ``Average case constant factor optimal multi-robot path planning in
  well-connected environments,'' \emph{arXiv preprint arXiv:1706.07255}, 2017,
  {note: A preliminary version appeared in the First International Symposium on
  Multi-Robot and Multi-Agent Systems, 2017}.

\bibitem{demaine2018coordinated}
E.~D. Demaine, S.~P. Fekete, P.~Keldenich, H.~Meijer, and C.~Scheffer,
  ``Coordinated motion planning: Reconfiguring a swarm of labeled robots with
  bounded stretch,'' \emph{arXiv preprint arXiv:1801.01689}, 2018.

\bibitem{WurDanMou08}
P.~R. Wurman, R.~D'Andrea, and M.~Mountz, ``Coordinating hundreds of
  cooperative, autonomous vehicles in warehouses,'' \emph{AI Magazine},
  vol.~29, no.~1, pp. 9--19, 2008.

\bibitem{stahlbock2008operations}
R.~Stahlbock and S.~Vo{\ss}, ``Operations research at container terminals: a
  literature update,'' \emph{OR spectrum}, vol.~30, no.~1, pp. 1--52, 2008.

\bibitem{tang2018hold}
S.~Tang, J.~Thomas, and V.~Kumar, ``Hold or take optimal plan (hoop): A
  quadratic programming approach to multi-robot trajectory generation,''
  \emph{The International Journal of Robotics Research}, p. 0278364917741532,
  2018.

\bibitem{ErdLoz86}
M.~A. Erdmann and T.~Lozano-P\'erez, ``On multiple moving objects,'' in
  \emph{Proceedings IEEE International Conference on Robotics \& Automation},
  1986, pp. 1419--1424.

\bibitem{LavHut98b}
S.~M. LaValle and S.~A. Hutchinson, ``Optimal motion planning for multiple
  robots having independent goals,'' \emph{IEEE Transactions on Robotics \&
  Automation}, vol.~14, no.~6, pp. 912--925, Dec. 1998.

\bibitem{GuoPar02}
Y.~Guo and L.~E. Parker, ``A distributed and optimal motion planning approach
  for multiple mobile robots,'' in \emph{Proceedings IEEE International
  Conference on Robotics \& Automation}, 2002, pp. 2612--2619.

\bibitem{JanStu08}
R.~Jansen and N.~Sturtevant, ``A new approach to cooperative pathfinding,'' in
  \emph{In International Conference on Autonomous Agents and Multiagent
  Systems}, 2008, pp. 1401--1404.

\bibitem{LunBer11}
R.~Luna and K.~E. Bekris, ``Push and swap: Fast cooperative path-finding with
  completeness guarantees,'' in \emph{Proceedings International Joint
  Conference on Artificial Intelligence}, 2011, pp. 294--300.

\bibitem{StaKor11}
T.~Standley and R.~Korf, ``Complete algorithms for cooperative pathfinding
  problems,'' in \emph{Proceedings International Joint Conference on Artificial
  Intelligence}, 2011, pp. 668--673.

\bibitem{BerSnoLinMan09}
J.~van~den Berg, J.~Snoeyink, M.~Lin, and D.~Manocha, ``Centralized path
  planning for multiple robots: Optimal decoupling into sequential plans,'' in
  \emph{Robotics: Science and Systems}, 2009.

\bibitem{SolHal12}
K.~Solovey and D.~Halperin, ``$k$-color multi-robot motion planning,'' in
  \emph{Proceedings Workshop on Algorithmic Foundations of Robotics}, 2012.

\bibitem{YuLav13STAR}
J.~Yu and S.~M. LaValle, ``Multi-agent path planning and network flow,'' in
  \emph{Algorithmic Foundations of Robotics {X}, Springer Tracts in Advanced
  Robotics}.\hskip 1em plus 0.5em minus 0.4em\relax Springer Berlin/Heidelberg,
  2013, vol.~86, pp. 157--173.

\bibitem{TurMicKum14}
M.~Turpin, K.~Mohta, N.~Michael, and V.~Kumar, ``{CAPT}: Concurrent assignment
  and planning of trajectories for multiple robots,'' \emph{International
  Journal of Robotics Research}, vol.~33, no.~1, pp. 98--112, 2014.

\bibitem{ChoLynHutKanBurKavThr05}
H.~Choset, K.~M. Lynch, S.~Hutchinson, G.~Kantor, W.~Burgard, L.~E. Kavraki,
  and S.~Thrun, \emph{Principles of Robot Motion: Theory, Algorithms, and
  Implementations}.\hskip 1em plus 0.5em minus 0.4em\relax Cambridge, MA: MIT
  Press, 2005.

\bibitem{blm-rvo}
J.~van~den Berg, M.~C. Lin, and D.~Manocha, ``Reciprocal velocity obstacles for
  real-time multi-agent navigation,'' in \emph{Proceedings IEEE International
  Conference on Robotics \& Automation}, 2008, pp. 1928--1935.

\bibitem{bekris2007decentralized}
K.~E. Bekris, K.~I. Tsianos, and L.~E. Kavraki, ``A decentralized planner that
  guarantees the safety of communicating vehicles with complex dynamics that
  replan online,'' in \emph{2007 IEEE/RSJ International Conference on
  Intelligent Robots and Systems}.\hskip 1em plus 0.5em minus 0.4em\relax IEEE,
  2007, pp. 3784--3790.

\bibitem{alonso2015local}
J.~Alonso-Mora, R.~Knepper, R.~Siegwart, and D.~Rus, ``Local motion planning
  for collaborative multi-robot manipulation of deformable objects,'' in
  \emph{2015 IEEE International Conference on Robotics and Automation
  (ICRA)}.\hskip 1em plus 0.5em minus 0.4em\relax IEEE, 2015, pp. 5495--5502.

\bibitem{knepper2012pedestrian}
R.~A. Knepper and D.~Rus, ``Pedestrian-inspired sampling-based multi-robot
  collision avoidance,'' in \emph{2012 IEEE RO-MAN: The 21st IEEE International
  Symposium on Robot and Human Interactive Communication}.\hskip 1em plus 0.5em
  minus 0.4em\relax IEEE, 2012, pp. 94--100.

\bibitem{HalLatWil00}
D.~Halperin, J.-C. Latombe, and R.~Wilson, ``A general framework for assembly
  planning: The motion space approach,'' \emph{Algorithmica}, vol.~26, no. 3-4,
  pp. 577--601, 2000.

\bibitem{Nna92}
B.~Nnaji, \emph{Theory of Automatic Robot Assembly and Programming}.\hskip 1em
  plus 0.5em minus 0.4em\relax Chapman \& Hall, 1992.

\bibitem{RodAma10}
S.~Rodriguez and N.~M. Amato, ``Behavior-based evacuation planning,'' in
  \emph{Proceedings IEEE International Conference on Robotics \& Automation},
  2010, pp. 350--355.

\bibitem{FoxBurKruThr00}
D.~Fox, W.~Burgard, H.~Kruppa, and S.~Thrun, ``A probabilistic approach to
  collaborative multi-robot localization,'' \emph{Autonomous Robots}, vol.~8,
  no.~3, pp. 325--344, Jun. 2000.

\bibitem{DinChaFai01}
J.~Ding, K.~Chakrabarty, and R.~B. Fair, ``Scheduling of microfluidic
  operations for reconfigurable two-dimensional electrowetting arrays,''
  \emph{IEEE Transactions on Computer-aided Design of Integrated Circuits and
  Systems}, vol.~20, no.~12, pp. 1463--1468, 2001.

\bibitem{GriAke05}
E.~J. Griffith and S.~Akella, ``Coordinating multiple droplets in planar array
  digital microfluidic systems,'' \emph{International Journal of Robotics
  Research}, vol.~24, no.~11, pp. 933--949, 2005.

\bibitem{MatNilSim95}
M.~J. Matari\'c, M.~Nilsson, and K.~T. Simsarian, ``Cooperative multi-robot box
  pushing,'' in \emph{Proceedings IEEE/RSJ International Conference on
  Intelligent Robots \& Systems}, 1995, pp. 556--561.

\bibitem{RusDonJen95}
D.~Rus, B.~Donald, and J.~Jennings, ``Moving furniture with teams of autonomous
  robots,'' in \emph{Proceedings IEEE/RSJ International Conference on
  Intelligent Robots \& Systems}, 1995, pp. 235--242.

\bibitem{JenWheEva97}
J.~S. Jennings, G.~Whelan, and W.~F. Evans, ``Cooperative search and rescue
  with a team of mobile robots,'' in \emph{Proceedings IEEE International
  Conference on Robotics \& Automation}, 1997.

\bibitem{SpiYak84}
P.~Spirakis and C.~K. Yap, ``Strong {NP}-hardness of moving many discs,''
  \emph{Information Processing Letters}, vol.~19, no.~1, pp. 55--59, 1984.

\bibitem{HopSchSha84}
J.~E. Hopcroft, J.~T. Schwartz, and M.~Sharir, ``On the complexity of motion
  planning for multiple independent objects; {PSPACE}-hardness of the
  ``warehouseman's problem'','' \emph{The International Journal of Robotics
  Research}, vol.~3, no.~4, pp. 76--88, 1984.

\bibitem{HeaDem05}
R.~A. Hearn and E.~D. Demaine, ``{PSPACE}-completeness of sliding-block puzzles
  and other problems through the nondeterministic constraint logic model of
  computation,'' \emph{Theoretical Computer Science}, vol. 343, no.~1, pp.
  72--96, 2005.

\bibitem{SolHal15}
K.~Solovey and D.~Halperin, ``On the hardness of unlabeled multi-robot motion
  planning,'' in \emph{Robotics: Science and Systems (RSS)}, 2015.

\bibitem{SolYu15}
K.~Solovey, J.~Yu, O.~Zamir, and D.~Halperin, ``Motion planning for unlabeled
  discs with optimality guarantees,'' in \emph{Robotics: Science and Systems},
  2015.

\bibitem{KorMilSpi84}
D.~Kornhauser, G.~Miller, and P.~Spirakis, ``Coordinating pebble motion on
  graphs, the diameter of permutation groups, and applications,'' in
  \emph{Proceedings IEEE Symposium on Foundations of Computer Science}, 1984,
  pp. 241--250.

\bibitem{Wil74}
R.~M. Wilson, ``Graph puzzles, homotopy, and the alternating group,''
  \emph{Journal of Combinatorial Theory (B)}, vol.~16, pp. 86--96, 1974.

\bibitem{AulMonParPer99}
V.~Auletta, A.~Monti, M.~Parente, and P.~Persiano, ``A linear-time algorithm
  for the feasbility of pebble motion on trees,'' \emph{Algorithmica}, vol.~23,
  pp. 223--245, 1999.

\bibitem{GorHas10}
G.~Goraly and R.~Hassin, ``Multi-color pebble motion on graph,''
  \emph{Algorithmica}, vol.~58, pp. 610--636, 2010.

\bibitem{YuRus15STAR}
J.~Yu and D.~Rus, ``Pebble motion on graphs with rotations: Efficient
  feasibility tests and planning,'' in \emph{Algorithmic Foundations of
  Robotics {XI}, Springer Tracts in Advanced Robotics}, vol. 107.\hskip 1em
  plus 0.5em minus 0.4em\relax Springer Berlin/Heidelberg, 2015, pp. 729--746.

\bibitem{Yu2016RAL}
J.~Yu, ``Intractability of optimal multi-robot path planning on planar
  graphs,'' \emph{IEEE Robotics and Automation Letters}, vol.~1, no.~1, pp.
  33--40, 2016.

\bibitem{ShaSteFelStu12}
G.~Sharon, R.~Stern, A.~Felner, and N.~Sturtevant, ``{Conflict-Based Search for
  Optimal Multi-Agent Path Finding},'' in \emph{Proc of the Twenty-Sixth AAAI
  Conference on Artificial Intelligence}, 2012.

\bibitem{WagChoC11}
G.~Wagner and H.~Choset, ``M*: A complete multirobot path planning algorithm
  with performance bounds,'' in \emph{Proceedings IEEE/RSJ International
  Conference on Intelligent Robots \& Systems}, 2011, pp. 3260--3267.

\bibitem{ferner2013odrm}
C.~Ferner, G.~Wagner, and H.~Choset, ``Odrm* optimal multirobot path planning
  in low dimensional search spaces,'' in \emph{Robotics and Automation (ICRA),
  2013 IEEE International Conference on}.\hskip 1em plus 0.5em minus
  0.4em\relax IEEE, 2013, pp. 3854--3859.

\bibitem{sharon2013increasing}
G.~Sharon, R.~Stern, M.~Goldenberg, and A.~Felner, ``The increasing cost tree
  search for optimal multi-agent pathfinding,'' \emph{Artificial Intelligence},
  vol. 195, pp. 470--495, 2013.

\bibitem{boyarski2015icbs}
E.~Boyarski, A.~Felner, R.~Stern, G.~Sharon, O.~Betzalel, D.~Tolpin, and
  E.~Shimony, ``Icbs: The improved conflict-based search algorithm for
  multi-agent pathfinding,'' in \emph{Eighth Annual Symposium on Combinatorial
  Search}, 2015.

\bibitem{honig2016multi}
W.~H{\"o}nig, T.~S. Kumar, L.~Cohen, H.~Ma, H.~Xu, N.~Ayanian, and S.~Koenig,
  ``Multi-agent path finding with kinematic constraints.'' in \emph{ICAPS},
  2016, pp. 477--485.

\bibitem{cohen2016improved}
L.~Cohen, T.~Uras, T.~Kumar, H.~Xu, N.~Ayanian, and S.~Koenig, ``Improved
  bounded-suboptimal multi-agent path finding solvers,'' in \emph{International
  Joint Conference on Artificial Intelligence}, 2016.

\bibitem{YuLav16TOR}
J.~Yu and S.~M. LaValle, ``Optimal multi-robot path planning on graphs:
  Complete algorithms and effective heuristics,'' \emph{IEEE Transactions on
  Robotics}, vol.~32, no.~5, pp. 1163--1177, 2016.

\bibitem{hall1935representatives}
P.~Hall, ``On representatives of subsets,'' \emph{Journal of the London
  Mathematical Society}, vol.~1, no.~1, pp. 26--30, 1935.

\bibitem{hopkroft1973n5}
J.~E. Hopcroft and R.~M. Karp, ``An $n^{\frac{5}{2}}$ algorithm for maximum
  matching in bipartite graphs,'' \emph{SIAM J. Comput}, vol.~2, pp. 225--231,
  1973.

\bibitem{SzeYu2017}
M.~Szegedy and J.~Yu, ``{The $n$-Stacks Problem},'' 2017, working manuscript.

\bibitem{cole2001edge}
R.~Cole, K.~Ost, and S.~Schirra, ``Edge-coloring bipartite multigraphs in o (e
  log d) time,'' \emph{Combinatorica}, vol.~21, no.~1, pp. 5--12, 2001.

\bibitem{goel2013perfect}
A.~Goel, M.~Kapralov, and S.~Khanna, ``Perfect matchings in
  o(n$\backslash$logn) time in regular bipartite graphs,'' \emph{SIAM Journal
  on Computing}, vol.~42, no.~3, pp. 1392--1404, 2013.

\bibitem{ford1956maximal}
L.~R. Ford and D.~R. Fulkerson, ``Maximal flow through a network,''
  \emph{Canadian journal of Mathematics}, vol.~8, no.~3, pp. 399--404, 1956.

\bibitem{HanRodYu2017Arxiv}
S.~D. Han, E.~J. Rodriguez, and J.~Yu, ``{SEAR: A Polynomial-Time Expected
  Constant-Factor Optimal Algorithmic Framework for Multi-Robot Path
  Planning},'' \emph{arXiv:1709.08215}, 2017.

\end{thebibliography}
\end{document}